\documentclass[11pt]{article}
\usepackage{arxiv_temp}
\usepackage{natbib}
\hypersetup{
    colorlinks,
    breaklinks,
    linkcolor = blue,
    citecolor = blue,
    urlcolor  = blue,
}
\usepackage{url}
\usepackage{paralist}
\usepackage{amsfonts}
\usepackage{nicefrac}
\usepackage{microtype}
\usepackage{graphicx,subfigure,color}
\usepackage{enumitem}
\usepackage{algorithmic,algorithm,bm}
\usepackage{amsmath,mathrsfs,amssymb,hyperref}
\usepackage{xcolor}
\usepackage{caption2}
\usepackage{dsfont}
\usepackage{booktabs,multirow}
\usepackage{standalone}     
\usepackage{preview} 
\usepackage{appendix}

\usepackage{amsthm}
\newtheorem{myThm}{Theorem}

\newtheorem{myLemma}{Lemma}

\theoremstyle{definition}
\newtheorem{myAssum}{Assumption}

\newtheorem{myInstance}{Instance}
\newtheorem{myRemark}{Remark}

\renewcommand{\hat}{\widehat}

\DeclareMathOperator*{\argmin}{arg\,min}

\usepackage{mathtools}
\let\norm\undefined 
\DeclarePairedDelimiter\norm{\lVert}{\rVert}
\DeclarePairedDelimiter\abs{\lvert}{\rvert}
\newcommand\inner[2]{\langle #1, #2 \rangle}
\allowdisplaybreaks

\def \x {\mathbf{x}}
\def \y {\mathbf{y}}

\def \u {\mathbf{u}}
\def \H {\mathcal{H}}

\def \R {\mathbb{R}}
\def \epsilon {\varepsilon}

\def \q {\mathbf{q}}

\def \p {\mathbf{p}}
\def \q {\mathbf{q}}

\def \xh {\widehat{\x}}
\def \xb {\bar{\x}}

\def \O {\mathcal{O}}

\def \T {\mathrm{T}}

\def \X {\mathcal{X}}

\def \Ecal {\mathcal{E}}

\def \p {\boldsymbol{p}}
\def \q {\boldsymbol{q}}
\def \ellb {\boldsymbol{\ell}}
\def \Lb {\boldsymbol{L}}
\def \m {\boldsymbol{m}}

\def \rg {\rangle}
\def \lg {\langle}
\def \expert {\mathtt{expert}\mbox{-}\mathtt{regret}}
\def \meta {\mathtt{meta}\mbox{-}\mathtt{regret}}

\firstpageno{1}

\begin{document}

\title{Dynamic Regret of Convex and Smooth Functions}

\author{\name Peng Zhao \email zhaop@lamda.nju.edu.cn \\
	\name Yu-Jie Zhang \email zhangyj@lamda.nju.edu.cn \\
	\name Lijun Zhang \email zhanglj@lamda.nju.edu.cn \\
    \name Zhi-Hua Zhou \email zhouzh@lamda.nju.edu.cn \\
    \addr National Key Laboratory for Novel Software Technology\\
    Nanjing University, Nanjing 210023, China}
\maketitle

\begin{abstract}
We investigate online convex optimization in non-stationary environments and choose the \emph{dynamic regret} as the performance measure, defined as the difference between cumulative loss incurred by the online algorithm and that of any feasible comparator sequence. Let $T$ be the time horizon and $P_T$ be the path-length that essentially reflects the non-stationarity of environments, the state-of-the-art dynamic regret is $\O(\sqrt{T(1+P_T)})$. Although this bound is proved to be minimax optimal for convex functions, in this paper, we demonstrate that it is possible to further enhance the dynamic regret by exploiting the smoothness condition. Specifically, we propose novel online algorithms that are capable of leveraging smoothness and replace the dependence on $T$ in the dynamic regret by \emph{problem-dependent} quantities: the variation in gradients of loss functions, the cumulative loss of the comparator sequence, and the minimum of the previous two terms. These quantities are at most $\O(T)$ while could be much smaller in benign environments. Therefore, our results are adaptive to the intrinsic difficulty of the problem, since the bounds are tighter than existing results for easy problems and meanwhile guarantee the same rate in the worst case.
\end{abstract}

\section{Introduction}
In many real-world applications, data are inherently accumulated over time, and thus it is of great importance to develop a learning system that updates in an online fashion. Online Convex Optimization (OCO) is a powerful paradigm for learning in such a circumstance, which can be regarded as an iterative game between a player and an adversary. At iteration $t$, the player selects a decision $\x_t$ from a convex set $\X$ and the adversary reveals a convex function $f_t: \X \mapsto \R$. The player subsequently suffers an instantaneous loss $f_t(\x_t)$. The performance measure is the (static) \emph{regret}~\citep{ICML'03:zinkvich},
\begin{equation}
  \label{eq:static-regret}
  \mbox{S-Regret}_T = \sum_{t=1}^T f_t(\x_t) - \min_{\x\in \mathcal{X}} \sum_{t=1}^T f_t(\x),
\end{equation}
which is the difference between cumulative loss incurred by the online algorithm and that of the best decision in hindsight. The rationale behind such a metric is that the best fixed decision in hindsight is reasonably good over all the iterations. However, this is too optimistic and may not hold in changing environments, where data are evolving and the optimal decision is drifting over time. To address this limitation, \emph{dynamic regret} is proposed to compete with changing comparators $\u_1,\dots,\u_T\in\mathcal{X}$,
\begin{equation}
  \label{eq:universal-dynamic-regret}
      \mbox{D-Regret}_T(\u_1,\dots,\u_T) = \sum_{t=1}^T f_t(\x_t) -  \sum_{t=1}^T f_t(\u_t),
\end{equation}
which draws considerable attention recently~\citep{OR'15:dynamic-function-VT,AISTATS'15:dynamic-optimistic,CDC'16:dynamic-sc,ICML'16:Yang-smooth,NIPS'17:zhang-dynamic-sc-smooth,ICML'18:zhang-dynamic-adaptive,NIPS'18:Zhang-Ader,COLT'19:dynamic-MAB,NIPS'19:Wangyuxiang,AAAI'20:Jianjun,AISTATS'20:BCO,arxiv:technique-note}. The measure is also called the \emph{universal} dynamic regret (or \emph{general} dynamic regret), in the sense that it gives a universal guarantee that holds against \emph{any} comparator sequence. Note that static regret~\eqref{eq:static-regret} can be viewed as its special form by setting comparators as the fixed best decision in hindsight. Moreover, a variant appeared frequently in the literature is the \emph{worst-case} dynamic regret defined as 
\begin{equation}
    \label{eq:worst-case-dynamic-regret}
    \mbox{D-Regret}_T(\x^*_1,\ldots,\x^*_T) = \sum_{t=1}^T f_t(\x_t) -  \sum_{t=1}^T f_t(\x^*_t),
\end{equation}
which specializes the general form~\eqref{eq:universal-dynamic-regret} by setting $\u_t = \x_t^* \in \argmin_{\x\in \mathcal{X}} f_t(\x)$. However, the worst-case dynamic regret is often too pessimistic, whereas the universal one is more adaptive to the non-stationary environments. We refer the readers to~\citep{NIPS'18:Zhang-Ader} for more detailed explanations.

There are many studies on the worst-case dynamic regret~\citep{OR'15:dynamic-function-VT,AISTATS'15:dynamic-optimistic,CDC'16:dynamic-sc,ICML'16:Yang-smooth,NIPS'17:zhang-dynamic-sc-smooth,ICML'18:zhang-dynamic-adaptive,NIPS'19:Wangyuxiang,arxiv:technique-note}, but only few results are known for the universal dynamic regret.~\citet{ICML'03:zinkvich} shows that online gradient descent (OGD) achieves an $\O(\sqrt{T}(1+P_T))$ universal dynamic regret, where $P_T = \sum_{t=2}^{T} \norm{\u_{t-1} - \u_{t}}_2$ is the path-length of comparators $\u_1,\ldots,\u_T$ and thus reflects the non-stationarity of the environments. Nevertheless, there exists a large gap between this upper bound and the $\Omega(\sqrt{T(1+P_T)})$ minimax lower bound established recently by~\citet{NIPS'18:Zhang-Ader}, who further propose a novel online algorithm, attaining an $\O(\sqrt{T(1+P_T)})$ universal dynamic regret, and thereby close the gap. 

Although the rate is minimax optimal for convex functions, we would like to design algorithms with more adaptive bounds, replacing the dependence on $T$ by certain \emph{problem-dependent} quantities that are $\O(T)$ in the worst case while could be much smaller in benign environments (i.e., easy problems). In the study of static regret, we can attain such bounds when additional curvature like smoothness is presented, including small-loss bounds~\citep{NIPS'10:smooth} and gradient-variation bounds~\citep{COLT'12:variation-Yang}. Thus, a natural question arises \emph{whether it is possible to leverage smoothness to achieve more adaptive universal dynamic regret?}

\paragraph{Our results.} In this paper, we provide an affirmative answer by designing online algorithms with problem-dependent dynamic regret bounds. Specifically, we focus on the following two adaptive quantities: the gradient variation of online functions $V_T$ and the cumulative loss of the comparator sequence $F_T$,
\begin{equation}
  \label{eq:gradient-variation}
  V_T = \sum_{t=2}^{T} \sup_{\x\in \X} \norm{\nabla f_{t-1}(\x) - \nabla f_t(\x)}_2^2, \text{  and  } F_T = \sum_{t=1}^{T} f_t(\u_t).
\end{equation}
We propose a novel online approach for convex and smooth functions, named \underline{S}moothness-a\underline{w}are \underline{o}nline lea\underline{r}ning with \underline{d}ynamic regret (abbreviated as \textsf{Sword}). There are three versions, including Sword$_{\text{var}}$, Sword$_{\text{small}}$, and Sword$_{\text{best}}$. All of them enjoy problem-dependent dynamic regret bound:
\begin{itemize}
    \item Sword$_{\text{var}}$ enjoys a gradient-variation bound of $\O(\sqrt{(1 + P_T + V_T)(1 + P_T)})$;
    \item Sword$_{\text{small}}$ enjoys a small-loss bound of $\O(\sqrt{(1 + P_T + F_T)(1 + P_T)})$;
    \item Sword$_{\text{best}}$ enjoys a best-of-both-worlds bound of $\O(\sqrt{(1+P_T + \min\{V_T,F_T\})(1+P_T)})$.
\end{itemize}
Comparing to the minimax rate of $\O(\sqrt{T(1+P_T)})$, our bounds replace the dependence on $T$ by the problem-dependent quantity $P_T + \min\{V_T,F_T\}$. Since the quantity is at most $\O(T)$, our bounds become much tighter when the problem is easy (for example when $P_T$ and $V_T/F_T$ are sublinear in $T$), and meanwhile safeguard the same guarantee in the worst case. Therefore, our results are adaptive to the intrinsic difficulty of the problem and the non-stationarity of the environments.  

\paragraph{Technical contributions.} We highlight challenges and technical contributions of this paper. First, we note that there exist studies showing that the worst-case dynamic regret can benefit from smoothness~\citep{ICML'16:Yang-smooth,NIPS'17:zhang-dynamic-sc-smooth,arxiv:technique-note}. However, their analyses do not apply to our case, since we cannot exploit the \emph{ optimality condition} of comparators $\u_1,\ldots,\u_T$, in stark contrast with the worst-case dynamic regret analysis. Therefore, we adopt the meta-expert framework to \emph{hedge the non-stationarity while keeping the adaptivity}. We can use variants of OGD as the expert-algorithm to exploit the smoothness, but it is difficult to design an appropriate meta-algorithm. Existing meta-algorithms and their variants either lead to  problem-independent regret bounds or introduce terms that are incompatible to the desired problem-dependent quantity. To address the difficulty, we adopt the technique of \emph{optimistic online learning}~\citep{conf/colt/RakhlinS13,NIPS'15:fast-rate-game}, in particular  OptimisticHedge, to design novel meta-algorithms.

For Sword$_{\text{var}}$, we apply OptimisticHedge with carefully designed optimism, which allows us to exploit the negative term in the regret analysis of OptimisticHedge~\citep{NIPS'15:fast-rate-game}. In this way, the meta-regret only depends on the gradient variation. The construction of the special optimism is the most challenging part of our paper. For Sword$_{\text{small}}$, the design of meta-algorithm is simple, and we directly use the vanilla Hedge, which can be treated as OptimisticHedge with null optimism. Finally, for Sword$_{\text{best}}$, we still employ OptimisticHedge as the meta-algorithm, but introduce a parallel meta-algorithm to \emph{learn the best optimism} to ensure a best-of-both-worlds dynamic regret guarantee.

\section{Related Work}
\label{sec:related-work}
We present a brief review of static and dynamic regret minimization for online convex optimization.

\subsection{Static Regret}
\label{sec:related-work-static-regret}
Static regret has been extensively studied in online convex optimization. Let $T$ be the time horizon and $d$ be the dimension, there exist online algorithms with static regret bounded by $\O(\sqrt{T})$, $\O(d\log T)$, and $\O(\log T)$ for convex, exponentially concave, and strongly convex functions, respectively~\citep{ICML'03:zinkvich,journals/ml/HazanAK07}. These results are proved to be minimax optimal~\citep{conf/colt/AbernethyBRT08}. More results can be found in the seminal books~\citep{book'16:Hazan-OCO,thesis:shai2007} and reference therein.

In addition to exploiting convexity of functions, there are studies improving static regret by incorporating smoothness, whose main proposal is to replace the dependence on $T$ by problem-dependent quantities. Such problem-dependent bounds enjoy much benign properties, in particular, they can safeguard the worst-case minimax rate yet can be much tighter in easy problem instances. In the literature, there are two kinds of such bounds, small-loss bounds~\citep{NIPS'10:smooth} and gradient variation bounds~\citep{COLT'12:variation-Yang}. 

Small-loss bounds are first introduced in the context of prediction with expert advice~\citep{journals/iandc/LittlestoneW94,JCSS'97:boosting}, which replace the dependence on $T$ by cumulative loss of the best expert. Later,~\citet{NIPS'10:smooth} show that in the online convex optimization setting, OGD can achieve an $\O(\sqrt{F^*_T})$ small-loss regret bound when the online convex functions are smooth and non-negative, where $F^*_T$ is the cumulative loss of the best decision in hindsight, namely, $F^*_T = \sum_{t=1}^{T} f_t(\x^*)$ with $\x^*$ chosen as the offline minimizer. The key ingredient in the analysis is to exploit the self-bounding properties of smooth functions.

Gradient variation bounds are introduced by~\citet{COLT'12:variation-Yang}, rooting in the development of second-order bounds for prediction with expert advice~\citep{COLT'05:second-order-Hedge} and online convex optimization~\citep{COLT'08:Hazan-variation}. For convex and smooth functions, \citet{COLT'12:variation-Yang} establish an $\O(\sqrt{V_T})$ static regret bound, where $V_T = \sum_{t=2}^{T} \sup_{\x\in \X} \norm{\nabla f_{t-1}(\x) - \nabla f_t(\x)}_2^2$ is the gradient variation. Gradient-variation bounds are particularly favored in slowly changing environments in which the online functions evolve gradually.

\subsection{Dynamic Regret}
\label{sec:related-work-dynamic-regret}
Dynamic regret enforces the player to compete with time-varying comparators, and thus is particularly favored in online learning in non-stationary environments~\citep{book/mit/sugiyama2012machine,TKDE'19:DFOP}. The notion of dynamic regret is also referred to as tracking regret or shifting regret in the prediction with expert advice setting~\citep{journals/ml/HerbsterW98,JMLR'01:Herbster}. It is known that in the worst case, sublinear dynamic regret is not attainable unless imposing certain regularities on the comparator sequence or the function sequence~\citep{OR'15:dynamic-function-VT,AISTATS'15:dynamic-optimistic}. The path-length is introduced by~\citet{ICML'03:zinkvich},
\begin{equation}
 	\label{eq:path-length}
 	P_T = \sum_{t=2}^{T} \norm{\u_{t-1} - \u_{t}}_2.
\end{equation}
Other regularities include the squared path-length introduced by~\citet{NIPS'17:zhang-dynamic-sc-smooth},
\begin{equation}
 	\label{eq:square-path-length}
 	S_T = \sum_{t=2}^{T} \norm{\u_{t-1} - \u_{t}}_2^2,
\end{equation}
and the function variation~\citep{OR'15:dynamic-function-VT}
\begin{equation}
 	\label{eq:function-variation}
 	V^f_T = \sum_{t=2}^{T} \sup_{\x\in \X} \abs{f_{t-1}(\x) - f_t(\x)}.
\end{equation} 

There are two kinds of dynamic regret in previous studies. The universal dynamic regret~\eqref{eq:universal-dynamic-regret} aims to compare with any feasible comparator sequence, while the worst-case dynamic regret specifies the comparator sequence to be the sequence of minimizers of online functions. In the following, we present related works respectively. Notice that we will use notations of $P_T$ and $S_T$ for path-length~\eqref{eq:path-length} and squared path-length~\eqref{eq:square-path-length} of the sequence $\{\u_t\}_{t=1,\ldots,T}$, while $P_T^*$ and $S_T^*$ for that of the sequence $\{\x_t^*\}_{t=1,\ldots,T}$ where $\x^*_t$ is the minimizer of the online function $f_t$, namely,
\begin{equation}
 	\label{eq:path-length-worst-case}
 	P_T^* = \sum_{t=2}^{T} \norm{\x^*_{t-1} - \x^*_{t}}_2, \text{ and  } S_T^* = \sum_{t=2}^{T} \norm{\x^*_{t-1} - \x^*_{t}}^2_2.
\end{equation} 

\paragraph{Universal dynamic regret.} The seminal work of~\citet{ICML'03:zinkvich} demonstrates that the online gradient descent (OGD) actually enjoys an $\O(\sqrt{T}(1+P_T))$ universal dynamic regret, and the regret guarantee holds against any feasible comparator sequence. Nevertheless, the result is far from the $\Omega(\sqrt{T(1+P_T)})$ lower bound established recently by~\citet{NIPS'18:Zhang-Ader}, who further close the gap by proposing a novel online algorithm that attains an optimal rate of $\O(\sqrt{T(1+P_T)})$ for convex functions~\citep{NIPS'18:Zhang-Ader}. Our work improve the minimax rate of $\O(\sqrt{T(1+P_T)})$ to problem-dependent regret guarantees by further exploiting the smoothness condition.

\paragraph{Worst-case dynamic regret.} More efforts of the dynamic regret analysis are devoted to studying the worst-case dynamic regret. \citet{ICML'16:Yang-smooth} prove that OGD enjoys an $\O(\sqrt{T(1 + P_T^*)})$ worst-case dynamic regret bound for convex functions when the path-length $P_T^*$ is known. For strongly convex and smooth functions,~\citet{CDC'16:dynamic-sc} show that an $\O(P_T^*)$ dynamic regret bound is achievable, and~\citet{NIPS'17:zhang-dynamic-sc-smooth} further propose the online multiple gradient descent algorithm and prove that the algorithm enjoys an $\O(\min\{P_T^*,S_T^*\})$ regret bound, which is recently enhanced to $\O(\min\{P_T^*,S_T^*, V_T^f\})$ by an improved analysis~\citep{arxiv:technique-note}. \citet{ICML'16:Yang-smooth} further show that $\O(P_T^*)$ rate is attainable for convex and smooth functions, provided that all the minimizers $\x_t^*$'s lie in the interior of the domain $\X$. The above results use the path-length (or squared path-length) as the regularity, which is in terms of the trajectory of comparator sequence. In another line of research, researchers use the variation with respect to the function values as the regularity. Specifically,~\citet{OR'15:dynamic-function-VT} show that OGD with a restarting strategy attains an $\O(T^{2/3}{V_T}^{f 1/3})$ regret for convex functions when the function variation $V^f_T$ is available, which is recently improved to $\O(T^{1/3}{V_T}^{f 2/3})$ for $1$-dim square loss~\citep{NIPS'19:Wangyuxiang}. 

\section{Gradient-Variation and Small-Loss Bounds}
We first list assumptions used in the paper, and then propose online algorithms with gradient-variation and small-loss dynamic regret bounds, respectively. At the end of this section, we present two concrete examples to illustrate the significance of the obtained problem-dependent bounds.

\subsection{Assumptions}
We introduce the following common assumptions that might be used in the theorems.
\begin{myAssum}
\label{assumption:bounded-gradient}
The norm of the gradients of online functions over the domain $\X$ is bounded by $G$, i.e., $\norm{\nabla f_t(\x)}_2 \leq G$, for all $\x \in \X$ and $t \in [T]$.
\end{myAssum}

\begin{myAssum}
\label{assumption:bounded-domain}
The domain $\X \subseteq \R^d$ contains the origin $\mathbf{0}$, and the diameter of the domain $\X$ is at most $D$, i.e., $\norm{\x -\x'}_2 \leq D$ for any $\x, \x' \in \X$.
\end{myAssum}

\begin{myAssum}
\label{assumption:smoothness}
All the online functions are $L$-smooth, i.e., for any $\x, \x' \in \X$ and $t \in [T]$,
\begin{equation} \label{eqn:f:smooth}
\norm{\nabla f_t(\x)-\nabla f_t(\x')}_2 \leq L \norm{\x-\x'}_2.
\end{equation}
\end{myAssum}

\begin{myAssum}
\label{assumption:non-negative}
All the online functions are non-negative.
\end{myAssum}

Note that in Assumption~\ref{assumption:non-negative} we require the online functions to be non-negative outside the domain $\X$, which is a precondition for establishing the self-bounding property for smooth functions~\citep{NIPS'10:smooth}. Meanwhile, we treat double logarithmic factors in $T$ as a constant, following previous studies~\citep{ALT'12:closer-adaptive-regret,COLT'15:Luo-AdaNormalHedge}.

\subsection{Gradient-Variation Bound}
\label{sec:gradient-variation-bound}
We design an approach in a meta-expert framework, and prove its gradient-variation dynamic regret. All the proofs can be found in Appendix~\ref{sec:appendix-variation}. 

\subsubsection{Expert-Algorithm}
In the study of static regret,~\citet{COLT'12:variation-Yang} propose the following online extra-gradient descent (OEGD) algorithm, and show that the algorithm enjoys gradient-variation static regret bound. The OEGD algorithm performs the following update:
\begin{equation}
  \label{alg:OEGD}
  \begin{split}
  \xh_{t+1} & = \Pi_{\X}\left[\xh_{t}-\eta\nabla f_{t}(\x_{t})\right],\\
  \x_{t+1} & = \Pi_{\X}\left[\xh_{t+1} - \eta \nabla f_{t}(\xh_{t+1})\right],
  \end{split}
\end{equation}
where $\xh_1, \x_1 \in \X$, $\eta>0$ is the step size, and $\Pi_{\X}[\cdot]$ denotes the projection onto the nearest point in $\X$. For convex and smooth functions,~\citet{COLT'12:variation-Yang} prove that OEGD achieves an $\O(\sqrt{V_T})$ static regret. We further demonstrate that OEGD also enjoys gradient-variation type dynamic regret.
\begin{myThm}
\label{thm:OEGD-dynamic-regret}
Under Assumptions~\ref{assumption:bounded-gradient},~\ref{assumption:bounded-domain}, and~\ref{assumption:smoothness}, by choosing $\eta \leq \frac{1}{4L}$, OEGD~\eqref{alg:OEGD} satisfies
\begin{equation*}
  \sum_{t=1}^T f_t(\x_{t}) - \sum_{t=1}^T f_t(\u_t) \leq \frac{D^2 + 2 D P_T}{2 \eta} + \eta V_T + GD = \O\Big(\frac{1+P_T}{\eta} + \eta V_T\Big).
\end{equation*}
for \emph{any} comparator sequence $\u_1,\ldots,\u_T \in \X$.
\end{myThm}
Theorem~\ref{thm:OEGD-dynamic-regret} shows that it is crucial to tune the step size to balance non-stationarity (path-length $P_T$) and adaptivity (gradient variation $V_T$). Notice that the optimal tuning $\eta^*= \sqrt{(D^2 + 2D P_T)/(2V_T)}$ requires the prior information of $P_T$ and $V_T$ that are generally unavailable. We emphasize that $V_T$ is empirically computable, while $P_T$ remains unknown even after all iterations due to the fact that the comparator sequence is unknown and can be chosen arbitrarily as long as it is feasible in the domain. Therefore, the doubling trick~\citep{JACM'97:doubling-trick} can only remove the dependence on the unknown $V_T$ but not $P_T$. 

To handle the uncertainty, we adopt the meta-expert framework to \emph{hedge the non-stationarity while keeping the adaptivity}, inspired by the recent advance in learning with multiple learning rates~\citep{COLT'14:second-order-Hedge,NIPS'16:MetaGrad,NIPS'18:Zhang-Ader}. Concretely, we first construct a pool of candidate step sizes to discretize value range of the optimal step size, and then initialize multiple experts simultaneously, denoted by $\Ecal_1,\ldots,\Ecal_N$. Each expert $\Ecal_i$ returns its prediction $\x_{t,i}$ by running OEGD~\eqref{alg:OEGD} with a step size $\eta_i$ from the pool. Finally, predictions of all the experts are combined by a meta-algorithm as the final output $\x_t$ to track the best expert. From the procedure, we observe that the dynamic regret can be decomposed as,
\begin{equation*}
  \mbox{D-Regret}_T = \sum_{t=1}^T f_t(\x_t) - \sum_{t=1}^{T} f_t(\u_t) = \underbrace{\sum_{t=1}^T f_t(\x_t) - \sum_{t=1}^T f_t(\x_{t,i})}_{\meta} + \underbrace{\sum_{t=1}^T f_t(\x_{t,i}) - \sum_{t=1}^T f_t(\u_t)}_{\expert},
\end{equation*}
where $\{\x_t\}_{t=1,\ldots,T}$ denotes the final output sequence, and $\{\x_{t,i}\}_{t=1,\ldots,T}$ is the prediction sequence of expert $\Ecal_i$. The first part is the difference between cumulative loss of final output sequence and that of prediction sequence of expert $\Ecal_i$, which is introduced by the meta-algorithm and thus named as \emph{meta-regret}; the second part is the dynamic regret of expert $\Ecal_i$ and therefore named as \emph{expert-regret}. 

The expert-algorithm is set as OEGD~\eqref{alg:OEGD}, and Theorem~\ref{thm:OEGD-dynamic-regret} upper bounds the expert-regret. The main difficulty lies in the design and analysis of an appropriate meta-algorithm.  

\subsubsection{Meta-Algorithm}
Formally, there are $N$ experts and expert $\Ecal_i$ predicts $\x_{t,i}$ at iteration $t$, and the meta-algorithm requires to produce $\x_t =\sum_{i=1}^{N} p_{t,i} \x_{t,i}$, a weighted combination of expert predictions, where $\p_{t} \in \Delta_N$ is the weight vector. It is natural to use Hedge~\citep{JCSS'97:boosting} for weight update in order to track the best expert.

In order to be compatible to the gradient-variation expert-regret, the meta-algorithm is required to incur a problem-dependent meta-regret of order $\O(\sqrt{V_T\ln N})$. However, the meta-algorithms used in existing studies~\citep{NIPS'16:MetaGrad,NIPS'18:Zhang-Ader} cannot satisfy the requirements. For example, the vanilla Hedge (multiplicative weights update) suffers from an $\O(\sqrt{T\ln N})$ meta-regret, which is problem-independent and thus not suitable for us. To this end, we design a a novel variant of Hedge by leveraging the technique of \emph{optimistic online learning} with carefully designed optimism, specifically for our problem.

\begin{algorithm}[!t]
   \caption{Sword$_{\text{var}}$: Meta-algorithm (VariationHedge)}
   \label{alg:VariationHedge-meta}
\begin{algorithmic}[1]
  \REQUIRE{step size pool $\H_{\text{var}} = \{\eta_i\}_{i=1}^{N}$ as specified in~\eqref{eq:step-size-pool-variation}; learning rate $\epsilon$}
  \STATE{Initialization: let $\x_{1}$ be any point in $\mathcal{X}$, and set $p_{0,i} = 1/N$ for $\forall i\in [N]$}
    \FOR{$t=1$ {\bfseries to} $T$}
      \STATE Receive the prediction $\x_{t+1,i}$ from expert $\Ecal_i$ (whose associated step size is $\eta_i$)
      \STATE Update the weight $p_{t+1,i}$ by~\eqref{eq:VariationHedge}
      \STATE Output the prediction $\x_{t+1} = \sum_{i=1}^{N} p_{t+1,i} \x_{t+1,i}$
    \ENDFOR
\end{algorithmic}
\end{algorithm}

\begin{algorithm}[!t]
   \caption{Sword$_{\text{var}}$: Expert-algorithm (OEGD)}
   \label{alg:variation-ogd-expert}
\begin{algorithmic}[1]
  \REQUIRE{step size $\eta_i$}
  \STATE{Let $\hat{\x}_{1,i}, \x_{1,i}$ be any point in $\mathcal{X}$}
    \FOR{$t=1$ {\bfseries to} $T$}
      \STATE $\xh_{t+1,i} = \Pi_{\mathcal{X}}\big[\xh_{t,i} - \eta_i \nabla f_{t}(\x_{t,i})\big]$   
      \STATE $\x_{t+1,i} = \Pi_{\mathcal{X}}\big[\xh_{t+1,i} - \eta_i \nabla f_{t}(\xh_{t+1,i})\big]$
      \STATE Send the prediction $\x_{t+1,i}$ to meta-algorithm     
    \ENDFOR
\end{algorithmic}
\end{algorithm}

The optimistic online learning is developed by~\citet{conf/colt/RakhlinS13} and further expanded by~\citet{NIPS'15:fast-rate-game}. For the prediction with expert advice setting, they consider that at the beginning of iteration $(t+1)$, in addition to the loss vector $\ellb_t \in \R^N$ returned by the experts, the learner can receive a vector $\m_{t+1} \in \R^N$ called \emph{optimism}. The authors propose the OptimisticHedge algorithm~\citep{conf/colt/RakhlinS13,NIPS'15:fast-rate-game}, which updates the weight vector $\p_{t+1} \in \Delta_{N}$ by
\begin{equation}
\label{eq:OptimisticHedge}
p_{t+1,i} \propto \exp\left(-\epsilon\Big(\sum_{s=1}^{t} \ell_{s,i} + m_{t+1,i}\Big)\right), \quad \forall i\in[N].
\end{equation}
\citet{NIPS'15:fast-rate-game} prove the following regret guarantee for OptimisticHedge.
\begin{myLemma}[{Theorem 19 of~\citet{NIPS'15:fast-rate-game}}]
\label{lemma:OptimisticHedge}
The meta-regret of OptimisticHedge is upper bounded by
\begin{align}
  \sum_{t=1}^{T} \inner{\p_t}{\ellb_t} - \ell_{t,i} \leq \frac{2 + \ln N}{\epsilon} + \epsilon \sum_{t=1}^{T} \norm{\ellb_t - \m_{t}}_{\infty}^2 - \frac{1}{4\epsilon}\sum_{t=2}^{T} \norm{\p_t - \p_{t-1}}_1^2,  \label{eq:regret-optimistic-Hedge}
\end{align}
which holds for any expert $i \in [N]$. Denote by $D_{\infty} = \sum_{t=1}^{T} \norm{\ellb_t - \m_{t}}_{\infty}^2$ to measure the adaptivity. With proper learning rate tuning, OptimisticHedge enjoys an $\O(\sqrt{D_\infty\ln N})$ meta-regret.
\end{myLemma}

The optimistic online learning is very powerful for designing adaptive methods, in that the adaptivity $D_\infty$ in Lemma~\ref{lemma:OptimisticHedge} is very general and can be specialized flexibly with different configurations of the feedback loss $\ellb_t$ and optimism $\m_t$. Based on the OptimisticHedge, we propose the \textsf{VariationHedge} algorithm as the meta-algorithm of Sword$_{\text{var}}$, by specializing OptimisticHedge as follows:
\begin{itemize}
    \item the feedback loss $\ellb_t$ is set as the linearized surrogate loss, namely, $\ell_{t,i} = \inner{\nabla f_t(\x_t)}{\x_{t,i}}$;
    \item the optimism $\m_t$ is set with a careful design: for each $i \in [N]$
    \begin{equation}
      \label{eq:optimism-variation}
      m_{t,i} = \inner{\nabla f_{t-1}(\bar{\x}_{t})}{\x_{t,i}}, \text{   where } \bar{\x}_{t} = \sum_{i=1}^{N} p_{t-1,i} \x_{t,i}.
    \end{equation}  
\end{itemize}
So the meta-algorithm of Sword$_{\text{var}}$ (namely, VariationHedge) updates the weight by
\begin{equation}
\label{eq:VariationHedge}
p_{t+1,i} \propto \exp\left(-\epsilon\Big(\sum_{s=1}^{t} \inner{\nabla f_s(\x_s)}{\x_{s,i}} + \inner{\nabla f_{t}(\bar{\x}_{t+1})}{\x_{t+1,i}}\Big)\right), \quad \forall i\in[N].
\end{equation}
Algorithm~\ref{alg:VariationHedge-meta} summarizes detailed procedures of the meta-algorithm, which in conjunction with the expert-algorithm of Algorithm~\ref{alg:variation-ogd-expert} yields the Sword$_{\text{var}}$ algorithm.

\begin{myRemark}
\label{remark:1}
The design of optimism in~\eqref{eq:optimism-variation} (in particular, $\xb_{t}$) is crucial, and is the most challenging part in this work. The key idea is to exploit the negative term in the regret of OptimisticHedge, as shown in~\eqref{eq:regret-optimistic-Hedge}, to convert the adaptive quantity $D_\infty$ to the desired gradient variation $V_T$. Indeed, 
\begin{align*}
   \norm{\ellb_t - \m_t}_{\infty}^2 \overset{\eqref{eq:optimism-variation}}{=} {} &  \max_{i\in[N]} \inner{\nabla f_t(\x_t) - \nabla f_{t-1}(\xb_t)}{\x_{t,i}}^2\\
   \leq {} & D^2 \norm{\nabla f_t(\x_t) - \nabla f_{t-1}(\xb_t)}_2^2 \\
   \leq {} & 2 D^2 (\norm{\nabla f_t(\x_t) - \nabla f_{t-1}(\x_t)}_2^2 + \norm{\nabla f_{t-1}(\x_t) - \nabla f_{t-1}(\xb_t)}_2^2) \\
   \leq {} & 2 D^2 \sup_{\x \in \X}\norm{\nabla f_t(\x) - \nabla f_{t-1}(\x)}_2^2 + 2 D^2 L^2\norm{\x_t - \xb_t}_2^2
\end{align*}
where the last step makes use of smoothness. Therefore, $D_\infty$ can be upper bounded by the gradient variation $V_T$ and the summation of $\norm{\x_t - \xb_t}_2^2$. The latter one can be further expanded as
\begin{align*}
  \norm{\x_t - \xb_t}_2^2 = \Big\Vert \sum_{i=1}^{N} (p_{t,i} - p_{t-1,i})\x_{t,i}\Big\Vert_2^2 \leq \Big(\sum_{i=1}^{N} \abs{p_{t,i} - p_{t-1,i}} \norm{\x_{t,i}}_2\Big)^2 \leq D^2 \norm{\p_t - \p_{t-1}}_1^2,
\end{align*}
which can be eliminated by the negative term in~\eqref{eq:regret-optimistic-Hedge}, with a suitable setting of the learning rate $\epsilon$.
\end{myRemark}

\subsubsection{Regret Guarantees}
We prove that the meta-regret of VariationHedge is $\O(\sqrt{V_T \ln N})$, compatible to the expert-regret.

\begin{myThm}
\label{thm:variation-meta-regret}
Under Assumptions~\ref{assumption:bounded-gradient},~\ref{assumption:bounded-domain}, and~\ref{assumption:smoothness}, by setting the learning rate optimally as $\epsilon = \min\{\sqrt{1/(8D^4L^2)},\sqrt{(2 + \ln N)/(2D^2V_T)}\}$, the meta-regret of VariationHedge is at most 
\begin{equation*}
  \meta \leq 2D\sqrt{2V_T(2+\ln N)} + 4\sqrt{2}D^2L(2+\ln N) = \O(\sqrt{V_T \ln N}).
\end{equation*}
\end{myThm}
Note that the dependence on $V_T$ in the optimal learning rate tuning can be removed by the doubling trick. Furthermore, actually we can set the optimal learning rate of the meta-algorithm with $\hat{V}_T = \sum_{t=2}^{T} \norm{\nabla f_t(\x_t) - \nabla f_{t-1}(\x_{t})}_2^2$ instead of the original gradient variation $V_T$ via a more refined analysis. The quantity $\hat{V}_T$ can be regarded as an empirical approximation of $V_T$, and it can be calculated directly without involving the inner problem of $\sup_{\x \in \X} \norm{\nabla f_t(\x) - \nabla f_{t-1}(\x)}_2^2$. Thereby, we can perform the doubling trick by monitoring $\hat{V}_T$ with much less computational efforts. Combining Theorem~\ref{thm:OEGD-dynamic-regret} (expert-regret) and Theorem~\ref{thm:variation-meta-regret} (meta-regret), we have the following dynamic regret bound.  
\begin{myThm}
\label{thm:dynamic-var}
Under Assumptions~\ref{assumption:bounded-gradient},~\ref{assumption:bounded-domain}, and~\ref{assumption:smoothness}, setting the pool of candidate step sizes $\H_{\text{var}}$ as 
\begin{equation}
  \label{eq:step-size-pool-variation}
  \H_{\text{var}} = \left\{\eta_i = 2^{i-1}\sqrt{\frac{D^2}{2GT}}, i \in [N_1]\right\},
\end{equation}
where $N_1 = \lceil 2^{-1} \log_2(GT/(8D^2L^2))\rceil + 1$.\footnote{The number of candidate step sizes is denoted by $N_1$ instead of $N$ to distinguish it with that of Sword$_{\text{small}}$.} Then Sword$_{\text{var}}$ (Algorithms~\ref{alg:VariationHedge-meta} and~\ref{alg:variation-ogd-expert}) satisfies
\begin{align*}
\label{eq:dynamic-regret-variation}
\sum_{t=1}^T f_t(\x_{t}) - \sum_{t=1}^T f_t(\u_t) \leq \O \Big(\sqrt{(1 + P_T + V_T)(1 + P_T)}\Big)
\end{align*}
for \emph{any} comparator sequence $\u_1,\ldots,\u_T \in \X$.
\end{myThm}
\begin{myRemark}
Compared with the existing $\O(\sqrt{T(1+P_T)})$ dynamic regret~\citep{NIPS'18:Zhang-Ader}, our result is more adaptive in the sense that it replaces $T$ by the \emph{problem-dependent} quantity $P_T + V_T$. Therefore, the bound will be much tighter in easy problems, for example when both $V_T$ and $P_T$ are $o(T)$. Meanwhile, it safeguards the same minimax rate, since both quantities are at most $\O(T)$.
\end{myRemark}

\begin{myRemark}
Because the \emph{universal} dynamic regret studied in this paper holds against any comparator sequence, it specializes the static regret by setting all comparators as the best fixed decision in hindsight, i.e., $\u_1=\ldots=\u_T=\x^* \in \argmin_{\x \in \X} \sum_{t=1}^{T} f_t(\x)$. Under such a circumstance, the path-length $P_T = \sum_{t=2}^{T} \norm{\u_{t-1} - \u_t}_2$ will be zero, so the regret bound in Theorem~\ref{thm:dynamic-var} actually implies an $\O (\sqrt{V_T})$ variation static regret bound, which recovers the result of~\citet{COLT'12:variation-Yang}.
\end{myRemark}
\subsection{Small-Loss Bound}
\label{sec:small-loss-bound}
In this part, we turn to another problem-dependent quantity, cumulative loss of the comparator sequence, and prove the small-loss dynamic regret. All the proofs can be found in Appendix~\ref{sec:appendix-small-loss}. 

We start from the online gradient descent (OGD),
\begin{equation}
  \label{eq:OGD}
  \x_{t+1} = \Pi_{\X}\big[\x_t - \eta \nabla f_t(\x_t)\big].
\end{equation}
\citet{NIPS'10:smooth} prove that OGD achieves an $\O(\sqrt{F^*_T})$ static regret, where $F^*_T = \sum_{t=1}^T f_t(\x^*)$ is the cumulative loss of the comparator benchmark $\x^*$. For the dynamic regret, since the benchmark is changing, a natural replacement is the cumulative loss of the comparator sequence $\u_1,\ldots,\u_T$, namely $F_T = \sum_{t=1}^T f_t(\u_t)$.  We show that OGD indeed enjoys such a small-loss dynamic regret.
\begin{myThm} 
\label{thm:dynamic-OGD}
Under Assumptions~\ref{assumption:bounded-domain},~\ref{assumption:smoothness}, and~\ref{assumption:non-negative}, by choosing any step size $\eta \leq \frac{1}{4L}$, OGD satisfies
\begin{align*}
  \sum_{t=1}^T f_t(\x_{t}) - \sum_{t=1}^T f_t(\u_t) \leq \frac{D^2 + 2DP_T}{2 \eta (1-2 \eta L) } + \frac{ 2\eta L}{1-2 \eta L} \sum_{t=1}^T f_t(\u_t) = \O \Big( \frac{1+P_T}{\eta} + \eta F_T \Big)
\end{align*}
for \emph{any} comparator sequence $\u_1,\ldots,\u_T \in \X$.
\end{myThm}
Similar to Sword$_{\text{var}}$, the step size needs to balance between non-stationarity ($P_T$) and adaptivity ($F_T$, this time). Notice that the optimal tuning depends on $P_T$ and $F_T$, both of which are unknown even after all $T$ iterations. Therefore, we again compensate the lack of this information via the meta-expert framework to hedge the non-stationarity while keeping the adaptivity. The expert-algorithm is set as OGD. The meta-algorithm is required to suffer a small-loss meta-regret of order $\O(\sqrt{F_T \ln N})$. We discover that vanilla Hedge with linearized surrogate loss is qualified, which updates the weight by
\begin{equation}
  \label{eq:vanilla-Hedge-surrogate}
  p_{t+1,i} \propto \exp\left(-\epsilon \sum_{s=1}^{t} \inner{\nabla f_s(\x_s)}{\x_{s,i}}\right), \quad \forall i\in[N].
\end{equation} 
Notice that vanilla Hedge can be treated as OptimisticHedge with null optimism, i.e., $\m_{t+1} = \bm{0}$. Therefore, by Lemma~\ref{lemma:OptimisticHedge} we know that its meta-regret is of order $\O(\sqrt{D_\infty \ln N})$ and 
\begin{equation}
    \label{eq:meta-regret-small-loss}
    D_\infty = \sum_{t=1}^{T} \max_{i\in[N]} \inner{\nabla f_t(\x_t)}{\x_{t,i}}^2 \leq D^2 \sum_{t=1}^{T} \norm{\nabla f_t(\x_t)}_2^2 \leq 4D^2L \sum_{t=1}^{T} f_t(\x_t),
\end{equation}
where the last inequality follows from the self-bounding property of smooth functions~\citep[Lemma 3.1]{NIPS'10:smooth}. As a result, the meta-regret is now $\O(\sqrt{F^{\x}_T \ln N})$, where $F^{\x}_T = \sum_{t=1}^{T} f_t(\x_t)$ is the cumulative loss of decisions. Note that the term $F^{\x}_T$ can be further processed to the desired small-loss quantity $F_T = \sum_{t=1}^{T} f_t(\u_t)$, the cumulative loss of comparators. We will present details in the proof. 

To summarize, Sword$_{\text{small}}$ chooses OGD~\eqref{eq:OGD} as the expert-algorithm, and uses the vanilla Hedge with linearized surrogate loss as the meta-algorithm shown in the update form~\eqref{eq:vanilla-Hedge-surrogate}. The theorem below shows that the proposed algorithm enjoys the small-loss dynamic regret bound.
\begin{myThm}
\label{thm:dynamic-small}
Under Assumptions~\ref{assumption:bounded-gradient},~\ref{assumption:bounded-domain},~\ref{assumption:smoothness}, and~\ref{assumption:non-negative}, setting the pool of candidate step sizes $\H_{\text{small}}$ as
\begin{equation}
  \label{eq:step-size-pool-small-loss}
  \H_{\text{small}} = \left\{\eta_i = 2^{i-1}\sqrt{\frac{D}{16LGT}}, i\in[N_2]\right\},
\end{equation}
where $N_2 = \lceil 2^{-1} \log_2(GT/(DL))\rceil + 1$. Setting the learning rate of meta-algorithm optimally as $\epsilon = \sqrt{(2+\ln N_2)/(D^2 F^{\x}_T)}$, then Sword$_{\text{small}}$ satisfies
\begin{align*}
\sum_{t=1}^T f_t(\x_{t}) - \sum_{t=1}^T f_t(\u_t) \leq \O \big(\sqrt{(1 + P_T + F_T)(1 + P_T)}\big).
\end{align*}
for \emph{any} comparator sequence $\u_1,\ldots,\u_T \in \X$.
\end{myThm}
Note that the optimal learning rate tuning requires the knowledge of $F^{\x}_T$, which can be easily removed by doubling trick or self-confident tuning~\citep{JCSS'02:Auer-self-confident}, since it is empirically evaluable at each iteration. Moreover, the $\O(\sqrt{(1 + P_T + F_T)(1 + P_T)})$ universal dynamic regret in Theorem~\ref{thm:dynamic-small} specializes to the $\O(\sqrt{F_T})$ static regret~\citep{NIPS'10:smooth} when setting the comparators as the fixed best decision in hindsight.

\subsection{Significance of Problem-Dependent Bounds}
\label{sec:example}
In this part, we justify the significance of our problem-dependent dynamic regret bounds. Specifically, we will present two concrete instances to demonstrate that it is possible to achieve a \emph{constant} dynamic regret bound instead of the minimax rate $\O(\sqrt{T(1+P_T)})$ by exploiting the problem's structure.

We consider the quadratic loss function of the form $f_t(x) = \frac{1}{2}(a_t \cdot x-b_t)^2$, where $a_t \neq 0$ and $x \in \X = [-1,1]$. Clearly, the function $f_t: \R \mapsto \R$ is convex and smooth. Denote by $T$ the time horizon. The coefficients $a_t$ and $b_t$ will be specified below in each instance.
\begin{myInstance}[{$V_T \ll F_T$}]
Let the time horizon $T = 2K+1$ be an odd with $K > 2$. We set the coefficients $a_t = 0.5 - \frac{t-1}{T}$ and $b_t =1 $ for all $t \in [T]$. 
\end{myInstance}
We set the comparator $u_t$ to be the minimizer of $f_t$, i.e, $u_t = x_t^* = \argmin_{x \in \X} f_t(x)$. Clearly, $u_t = 1$ for $t \in [K+1]$, and $u_t = -1$ for $t = K+2,\ldots,T$. Therefore, we have
\begin{align*}
V_T ={}& \sum_{t=2}^T\sup_{x\in\X}\vert(a_{t-1}^2-a_t^2)x-(a_{t-1}-a_t)\vert^2 \sum_{t=2}^T\sup_{x\in\X}\left\vert\left(\frac{T-2t+3}{T^2}\right)\cdot x-\frac{1}{T}\right\vert^2\\
={}&\sum_{t=2}^{K+2}\left(\frac{2T-(2t-3)}{T^2}\right)^2+\sum_{t=K+3}^T\left(\frac{2t-3}{T^2}\right)^2 \leq \sum_{t=2}^T \left(\frac{2}{T}\right)^2 = \O(1).
\end{align*}
\begin{align*}
	F_T = {} & \sum_{t=1}^{T} \frac{1}{2}(a_t u_t -b_t)^2 = \sum_{t=1}^{K+1} \frac{1}{2}\left(0.5-\frac{t-1}{T} - 1\right)^2 + \sum_{t=K+2}^{T} \frac{1}{2}\left(-0.5+\frac{t-1}{T} - 1\right)^2 = \Theta(T).
\end{align*}
We can observe that $V_T \leq \O(1)$ is significantly smaller than $F_T = \Theta(T)$ (as well as the problem-independent quantity $T$)  in this instance. Meanwhile, the path-length term $P_T = \O(1)$. As a result, the minimax dynamic regret bound is $\O(\sqrt{T(1+P_T)}) = \O(\sqrt{T})$; the small-loss bound is $\O(\sqrt{(1+P_T+F_T)(1+P_T)}) = \O(\sqrt{T})$; and the gradient-variation bound is $\O(\sqrt{(1+P_T+V_T)(1+P_T)}) = \O(1)$. In other words, by exploiting the problem's structure, our approach (Sword$_{\text{var}}$) can enjoy a \emph{constant} dynamic regret in this scenario.

\begin{myInstance}[{$F_T \ll V_T$}]
Let the time horizon $T = 2K$ be an even. During the first half iterations, $(a_t,b_t)$ is set as $(1,1)$ on odd rounds and $(0.5,0.5)$ on even rounds. During the remaining iterations, $(a_t,b_t)$ is set as $(1,-1)$ on odd rounds and $(0.5,-0.5)$ on even rounds.
\end{myInstance}
We set the comparator $u_t$ to be the minimizer of $f_t$, i.e, $u_t = x_t^* = \argmin_{x \in \X} f_t(x)$. Clearly, $u_t = 1$ for $t \in [K]$, and $u_t = -1$ for $t = K+1,\ldots,T$. Therefore, we have
\begin{align*}
  	V_T =\sum_{t=2}^{T} \sup_{x\in \X} \abs{(a_{t-1}^2 - a_t^2)x - (a_{t-1}b_{t-1} - a_tb_t)}^2 = \Theta(T),\qquad F_T = 0.
\end{align*}
We can see that $F_T = 0$ is considerably smaller than $V_T = \Theta(T)$ (as well as the problem-independent quantity $T$) in this scenario. Meanwhile, the path-length term $P_T = \O(1)$. As a result, the minimax dynamic regret bound is $\O(\sqrt{T(1+P_T)}) = \O(\sqrt{T})$; the gradient-variation bound is $\O(\sqrt{(1+P_T+V_T)(1+P_T)}) = \O(\sqrt{T})$; and the small-loss bound is $\O(\sqrt{(1+P_T+F_T)(1+P_T)}) = \O(1)$. In other words, by exploiting the problem's structure, our approach (Sword$_{\text{small}}$) can enjoy a \emph{constant} dynamic regret in this scenario. 

\section{Best-of-Both-Worlds Bound}
\label{sec:bobw-bound}
In the last section, we propose Sword$_\text{var}$ and Sword$_\text{small}$ that achieve gradient-variation and small-loss bounds respectively. Due to different problem-dependent quantities are involved, these two bounds are generally incomparable and are favored in different scenarios, as demonstrated by the concrete examples in Section~\ref{sec:example}. Therefore, it is natural to ask for a \emph{best-of-both-worlds} guarantee: the regret of the minimum of gradient-variation and small-loss bounds.

To this end, we require a meta-algorithm that can enjoy both kinds of adaptivity to combine all the experts, with an $\O(\sqrt{\min\{V_T,F_T\}\ln N})$ meta-regret. Based on the observation that \emph{both VariationHedge and vanilla Hedge are essentially special cases of OptimisticHedge with different configurations of optimism}, we adopt the OptimisticHedge to be the meta-algorithm for Sword$_{\text{best}}$, where a parallel meta-algorithm  is introduced to \emph{learn the best optimism} for OptimisticHedge to ensure best-of-both-worlds meta-regret. In the following we describe the expert-algorithm and meta-algorithm of Sword$_{\text{best}}$.

\paragraph{Expert-algorithm.} We aggregate the experts of Sword$_{\text{var}}$ and Sword$_{\text{small}}$, so there are $N= N_1 + N_2$ experts in total and the step size of each experts is set according to the pool $\H = \H_{\text{var}} \cup \H_{\text{small}}$ (cf.~\eqref{eq:step-size-pool-variation} and~\eqref{eq:step-size-pool-small-loss}  for definitions). The first $N_1$ experts run OEGD~\eqref{alg:OEGD} with the step size chosen from $\H_{\text{var}}$, and the other $N_2$ experts perform OGD~\eqref{eq:OGD} with step size specified by $\H_{\text{small}}$. At iteration $t$, the final output is a weighted combination of predictions returned by the expert-algorithms, namely,
\begin{equation}
  \label{eq:best-output}
  \x_{t} = \sum_{i=1}^{N} p_{t,i}\x_{t,i} = \sum_{i=1}^{N_1} p_{t,i}\x_{t,i}^{v} + \sum_{i=N_1 + 1}^{N_1 + N_2} p_{t,i}\x_{t,i}^{s},
\end{equation}
where $\p_t \in \Delta_{N_1 + N_2}$ is the weight, $\x_{t,i} = \x_{t,i}^{v}$ for $i=1,\ldots,N_1$ are predictions returned by the expert-algorithms (OEGD) of Sword$_{\text{var}}$, and $\x_{t,i} = \x_{t,i}^{s}$ for $i=N_1 + 1,\ldots,N_1 + N_2$ are predictions returned by the expert-algorithms (OGD) of Sword$_{\text{small}}$. It remains to specify the meta-algorithm.

\begin{table}[!t]
\caption{Summary of expert-algorithms and meta-algorithms as well as different optimism used in the proposed algorithms (including three variants of Sword).}
\vspace{2mm}
\centering
\label{table:meta-expert-summary}
\resizebox{0.7\textwidth}{!}{
\begin{tabular}{lcccc}\toprule
\multicolumn{1}{c}{\textbf{Method}}   & \textbf{Expert} & \textbf{Meta} & \textbf{Optimism}\\ \midrule
Sword$_{\text{var}}$  & OEGD  & VariationHedge & by~\eqref{eq:optimism-variation}\\ 
Sword$_{\text{small}}$  &  OGD  & vanilla Hedge & $\m_{t+1} = \bm{0}$\\ 
Sword$_{\text{best}}$ &  OEGD \& OGD  & OptimisticHedge & by~\eqref{eq:BEST-optimism},~\eqref{eq:setting-optimism-best}\\ 
\bottomrule
\end{tabular}
}
\end{table}

\paragraph{Meta-algorithm.} We adopt the OptimisticHedge algorithm along with the linearized surrogate loss as the meta-algorithm, where the weight vector $\p_{t+1} \in \Delta_{N_1 + N_2}$ is updated according to 
\begin{equation}
  \label{eq:BEST-OptimisticHedge}
  p_{t+1,i} \propto \exp\left(-\epsilon \Big(\sum_{s=1}^{t} \inner{\nabla f_s(\x_s)}{\x_{s,i}} + m_{t+1,i}\Big)\right),
\end{equation}
where the optimism $\m_{t+1} \in \R^{N_1+N_2}$. In order to facilitate the meta-algorithm with both kinds of adaptivity ($V_T$ and $F_T$), it is crucial to design best-of-both-worlds optimism.

We set the optimism $\m_{t+1}$ in the following way: for each $i \in [N_1 + N_2]$
\begin{equation}
  \label{eq:BEST-optimism}
  m_{t+1,i} = \inner{M_{t+1}}{\x_{t+1,i}},
\end{equation}
where $M_{t+1} \in \R^d$ is called the optimistic vector. So we are left with the task of determining the term of $M_{t+1}$ in~\eqref{eq:BEST-optimism}. Inspired by the seminal work of~\citet{conf/colt/RakhlinS13}, we treat the problem of selecting the sequence of optimistic vectors as another online learning problem. The idea is to build a parallel meta-algorithm for learning the optimistic vector $M_{t+1}$, which is then fed to OptimisticHedge of~\eqref{eq:BEST-OptimisticHedge} for combining multiple experts, to achieve a best-of-both-worlds meta-regret.

Specifically, consider the following learning scenario of \emph{prediction with two expert advice}. At the beginning of iteration $(t+1)$, we receive two optimistic vectors $M_{t+1}^{v}, M_{t+1}^{s} \in \R^d$, based on which the algorithm determines the optimistic vector $M_{t+1} \in \R^d$ for Sword$_\text{best}$. Then the online function $f_{t+1}$ is revealed, and we subsequently observe the loss of $d_{t+1}(M_{t+1}^{v})$ and $d_{t+1}(M_{t+1}^{s})$, where $d_{t+1}(M) = \norm{\nabla f_{t+1}(\x_{t+1}) - M}_2^2$. In above, the vectors of $M_{t+1}^{v}$ and $M_{t+1}^{s}$ are
\begin{equation}
  \label{eq:M_t-variation}
  M_{t+1}^v = \nabla f_{t}(\xb_{t+1}),~\text{ and }~M_{t+1}^s = \bm{0},
\end{equation}
where $\xb_{t+1}$ is the instrumental output. Similar to the construction of~\eqref{eq:optimism-variation}, it is designed as 
\begin{equation}
  \label{eq:x-bar-best}
  \xb_{t+1} = \sum_{i=1}^{N_1} p_{t,i}\x_{t+1,i}^{v} + \sum_{i=N_1 + 1}^{N_1 + N_2} p_{t,i}\x_{t+1,i}^{s}.
\end{equation} 
Notice that the function $d_t: \R^d \mapsto \R$ is 2-strongly convex with respect to $\Vert\cdot\Vert_2$-norm, we thus choose Hedge of strongly convex functions~\citep[Chapter 3.3]{book/Cambridge/cesa2006prediction} as the parallel meta-algorithm for updating, 
\begin{equation}
  \label{eq:setting-optimism-best}
  M_{t+1} = \beta_{t+1} M_{t+1}^{v} + (1-\beta_{t+1}) M_{t+1}^{s},
\end{equation}
where the weight $\beta_{t+1} \in [0,1]$ for learning optimistic vectors is updated by 
\begin{equation}
    \label{eq:sc-weight-update}
    \beta_{t+1} = \frac{\exp(-2 D_t^v)}{\exp(-2 D_t^v) + \exp(-2 D_t^s )}
\end{equation}
with $D_t^v = \sum_{\tau=1}^{t} d_\tau(M_\tau^{v})$ and $D_t^s = \sum_{\tau=1}^{t} d_\tau(M_\tau^{s})$.

\begin{algorithm}[!t]
   \caption{Sword$_{\text{best}}$: Meta-algorithm (OptimisticHedge)}
   \label{alg:OptimisticHedge-meta}
\begin{algorithmic}[1]
  \REQUIRE{step size pool $\H = \{\eta_i\}_{i=1}^N$ as specified in~\eqref{eq:step-size-pool-best}; learning rate $\epsilon$}
  \STATE{Initialization: let $\x_{1}$ be any point in $\mathcal{X}$; set $N = N_1 + N_2$ and $p_{0,i} = 1/N$ for $\forall i\in [N]$}
    \FOR{$t=1$ {\bfseries to} $T$}
      \STATE Receive the prediction $\x_{t+1,i}$ from expert $\Ecal_i$\\ 
      \texttt{\% learning the optimism}
      \STATE Set $M_{t+1}^v$ and $M_{t+1}^s$ by~\eqref{eq:M_t-variation} and~\eqref{eq:x-bar-best}
      \STATE Update the weight $\beta_{t+1}$ by~\eqref{eq:sc-weight-update}
      \STATE Obtain the optimism $M_{t+1}$~\eqref{eq:setting-optimism-best}\\ 
      \texttt{\% back to OptimisticHedge}
      \STATE Update the weight $p_{t+1,i}$ by~\eqref{eq:BEST-OptimisticHedge} and~\eqref{eq:BEST-optimism}
      \STATE Output the prediction $\x_{t+1} = \sum_{i=1}^{N} p_{t+1,i} \x_{t+1,i}$
    \ENDFOR
\end{algorithmic}
\end{algorithm}

\begin{algorithm}[!t]
   \caption{Sword$_{\text{best}}$: Expert (OEGD \& OGD)}
   \label{alg:BEST-expert}
\begin{algorithmic}[1]
  \REQUIRE{step size $\eta_i$}
  \STATE{Let $\hat{\x}_{1,i},\x_{1,i}$ be any point in $\mathcal{X}$}
    \FOR{$t=1$ {\bfseries to} $T$}
      \IF{$i \in \{1,\ldots,N_1\}$}
      \STATE $\xh_{t+1,i} = \Pi_{\mathcal{X}}\big[\xh_{t,i} - \eta_i \nabla f_{t}(\x_{t,i})\big]$   
        \STATE $\x_{t+1,i} = \Pi_{\mathcal{X}}\big[\xh_{t+1,i} - \eta_i \nabla f_{t}(\xh_{t+1,i})\big]$.      
      \ELSE
        \STATE $\x_{t+1,i} = \Pi_{\mathcal{X}}\big[\x_{t,i} - \eta_i \nabla f_{t}(\x_{t,i})\big]$.
      \ENDIF      
      \STATE Send the prediction $\x_{t+1,i}$ to meta-algorithm     
    \ENDFOR
\end{algorithmic}
\end{algorithm}
Algorithm~\ref{alg:OptimisticHedge-meta} summarizes the meta-algorithm of Sword$_{\text{best}}$, and Algorithm~\ref{alg:BEST-expert} further shows the expert-algorithm. In the last two columns of Table~\ref{table:meta-expert-summary}, we present comparisons of the  meta-algorithms and optimism designed for different methods.

\paragraph{Regret Analysis.} Recall that the meta-regret of OptimisticHedge is of order $\O(\sqrt{D_{\infty}\ln N})$. From the setting of surrogate loss~\eqref{eq:BEST-OptimisticHedge} and  optimism~\eqref{eq:BEST-optimism}, we have
\begin{align*}
  D_{\infty} = \sum_{t=1}^{T} \max_{i\in[N]} \left(\inner{\nabla f_t(\x_t) - M_t}{\x_{t,i}}\right)^2 \leq D^2 \sum_{t=1}^{T} \norm{\nabla f_t(\x_t) - M_t}_2^2.
\end{align*}
Besides, the regret analysis of Hedge for strongly convex functions~\citep[Proposition 3.1]{book/Cambridge/cesa2006prediction} implies
\[
    \sum_{t=1}^{T} \norm{\nabla f_t(\x_t) - M_t}_2^2 = \sum_{t=1}^{T} d_t(M_t) \leq \min \big\{\bar{V}_T, \bar{F}_T\big\} + \frac{\ln 2}{2},
\]
where $\bar{V}_T = \sum_{t=2}^{T} \norm{\nabla f_t(\x_t) - \nabla f_{t-1}(\xb_{t})}_2^2$ and $\bar{F}_T = \sum_{t=1}^{T} \norm{\nabla f_t(\x_t)}_2^2$. The two terms can be further converted to the desired gradient variation $V_T$ and small loss $F_T$, by exploiting the smoothness and expert-regret analysis. We can thus ensure the following meta-regret bound, whose proof can be found in Appendix~\ref{sec:appendix-bobw}.
\begin{myThm}
\label{thm:BEST-meta-regret}
Under Assumptions~\ref{assumption:bounded-gradient},~\ref{assumption:bounded-domain},~\ref{assumption:smoothness}, and~\ref{assumption:non-negative}, by setting the learning rate optimally as $\epsilon = \min\{\sqrt{1/(8D^4L^2)},\epsilon^*\}$, the meta-algorithm of Sword$_{\text{best}}$ satisfies
\begin{align*}
  \meta \leq 2D \sqrt{(2+\ln N)(\min\{2V_T,\bar{F}_T\} +  \ln 2)} + 4\sqrt{2}D^2L (2+\ln N)
\end{align*}
where $\epsilon^* = \sqrt{(2 + \ln N)/(D^2\min\{2V_T,\bar{F}_T\} + D^2\ln 2)}$.
\end{myThm}
Because $V_T$ and $\bar{F}_T$ are both empirically observable, we can easily get rid of their dependence in the optimal learning rate tuning. Also see the discussion below Theorem~\ref{thm:variation-meta-regret} about replacing the original gradient variation $V_T$ by its empirical approximation $\hat{V}_T = \sum_{t=2}^{T} \norm{\nabla f_t(\x_t) - \nabla f_{t-1}(\x_{t})}_2^2$ to save computational costs. Besides, the $\bar{F}_T$ term of meta-regret will be converted to the desired small-loss quantity $F_T$ in the final regret bound. 

Combining above meta-regret analysis as well as the expert-regret analysis of OEGD and OGD algorithms, we can finally achieve the best of both worlds. Appendix~\ref{sec:appendix-bobw} presents the proof.
\begin{myThm}
\label{thm:dynamic-best}
Under Assumptions~\ref{assumption:bounded-gradient},~\ref{assumption:bounded-domain},~\ref{assumption:smoothness}, and~\ref{assumption:non-negative}, setting the pool of candidate step sizes as 
\begin{equation}
  \label{eq:step-size-pool-best}
  \H = \H_{\text{var}} \cup \H_{\text{small}},
\end{equation}
where $\H_{\text{var}}$ and $\H_{\text{small}}$ are defined in~\eqref{eq:step-size-pool-variation} and~\eqref{eq:step-size-pool-small-loss}. Then Sword$_{\text{best}}$ (Algorithms~\ref{alg:OptimisticHedge-meta} and~\ref{alg:BEST-expert}) satisfies
\begin{equation*}
\label{eq:best-of-both-worlds}
\sum_{t=1}^T f_t(\x_{t}) - \sum_{t=1}^T f_t(\u_t) \leq \O\big(\sqrt{(1 + P_T + \min\{V_T,F_T\})(1+P_T)}\big),
\end{equation*}
for \emph{any} comparator sequence $\u_1,\ldots,\u_T \in \X$.
\end{myThm}

\begin{myRemark}
The dynamic regret bound in Theorem~\ref{thm:dynamic-best} achieves a minimum of gradient-variation and small-loss bounds, and therefore combines their advantages and enjoys both kinds of adaptivity. Moreover, the result also implies an $\O(\sqrt{\min\{V_T, F_T\}})$ static regret by setting the sequence of comparators to be the best fixed decision in hindsight, where we note that $F_T$ is now the same as the notation of $F_T^*$ used below~\eqref{eq:OGD}, the cumulative loss of the comparators benchmark.
\end{myRemark}

\section{Conclusion}
In this paper, we exploit smoothness to enhance the universal dynamic regret, with the aim to replace the time horizon $T$ in the state-of-the-art $\O(\sqrt{T(1+P_T)})$ bound by \emph{problem-dependent} quantities that are at most $\O(T)$ but can be much smaller in easy problems. We achieve this goal by proposing two meta-expert algorithms: Sword$_{\text{var}}$ which attains a gradient-variation dynamic regret bound of order $\O(\sqrt{(1+P_T + V_T)(1+P_T)})$, and Sword$_{\text{small}}$ which enjoys a small-loss dynamic regret bound of order $\O(\sqrt{(1+P_T + F_T)(1+P_T)})$. Here, $V_T$ measures the variation in gradients and $F_T$ is the cumulative loss of the comparator sequence. They are at most $\O(T)$ yet could be very small when the problem is easy, and thus reflect the difficulty of the problem instance. As a result, our bounds improve the minimax rate of universal dynamic regret by exploiting smoothness. Finally, we design Sword$_{\text{best}}$ to combine advantages of both gradient-variation and small-loss algorithms and achieve a best-of-both-worlds dynamic regret bound of order $\O(\sqrt{(1 + P_T + \min\{V_T,F_T\})(1+P_T)})$. We note that all of attained dynamic regret bounds are universal in the sense that they hold against \emph{any} feasible comparator sequence, and thus the algorithms are more adaptive to the non-stationary environments. In the future, we will investigate the possibility of exploiting other function curvatures, such as strong convexity or exp-concavity, into the analysis of the universal dynamic regret.

\section*{Acknowledgment}
This research was supported by the National Key R\&D Program of China (2018YFB1004300), NSFC (61921006, 61976112), the Collaborative Innovation Center of Novel Software Technology and Industrialization, and the Baidu Scholarship. The authors would like to thank Mengxiao Zhang for helpful discussions. We are also grateful for
the anonymous reviewers for their valuable comments.

\bibliography{../../Reference/online_learning}

\begin{thebibliography}{36}
\providecommand{\natexlab}[1]{#1}
\providecommand{\url}[1]{\texttt{#1}}
\expandafter\ifx\csname urlstyle\endcsname\relax
  \providecommand{\doi}[1]{doi: #1}\else
  \providecommand{\doi}{doi: \begingroup \urlstyle{rm}\Url}\fi

\bibitem[Abernethy et~al.(2008)Abernethy, Bartlett, Rakhlin, and
  Tewari]{conf/colt/AbernethyBRT08}
Jacob~D. Abernethy, Peter~L. Bartlett, Alexander Rakhlin, and Ambuj Tewari.
\newblock Optimal stragies and minimax lower bounds for online convex games.
\newblock In \emph{Proceedings of the 21st Annual Conference on Learning Theory
  (COLT)}, pages 415--424, 2008.

\bibitem[Adamskiy et~al.(2012)Adamskiy, Koolen, Chernov, and
  Vovk]{ALT'12:closer-adaptive-regret}
Dmitry Adamskiy, Wouter~M. Koolen, Alexey~V. Chernov, and Vladimir Vovk.
\newblock A closer look at adaptive regret.
\newblock In \emph{Proceedings of the 23rd International Conference on
  Algorithmic Learning Theory (ALT)}, pages 290--304, 2012.

\bibitem[Auer et~al.(2002)Auer, Cesa{-}Bianchi, and
  Gentile]{JCSS'02:Auer-self-confident}
Peter Auer, Nicol{\`{o}} Cesa{-}Bianchi, and Claudio Gentile.
\newblock Adaptive and self-confident on-line learning algorithms.
\newblock \emph{Journal of Computer and System Sciences}, 64\penalty0
  (1):\penalty0 48--75, 2002.

\bibitem[Auer et~al.(2019)Auer, Chen, Gajane, Lee, Luo, Ortner, and
  Wei]{COLT'19:dynamic-MAB}
Peter Auer, Yifang Chen, Pratik Gajane, Chung-Wei Lee, Haipeng Luo, Ronald
  Ortner, and Chen-Yu Wei.
\newblock Achieving optimal dynamic regret for non-stationary bandits without
  prior information.
\newblock In \emph{Proceedings of the 32nd Conference on Learning Theory
  (COLT)}, pages 159--163, 2019.

\bibitem[Baby and Wang(2019)]{NIPS'19:Wangyuxiang}
Dheeraj Baby and Yu-Xiang Wang.
\newblock Online forecasting of total-variation-bounded sequences.
\newblock In \emph{Advances in Neural Information Processing Systems 32
  (NeurIPS)}, pages 11071--11081, 2019.

\bibitem[Besbes et~al.(2015)Besbes, Gur, and Zeevi]{OR'15:dynamic-function-VT}
Omar Besbes, Yonatan Gur, and Assaf~J. Zeevi.
\newblock Non-stationary stochastic optimization.
\newblock \emph{Operations Research}, 63\penalty0 (5):\penalty0 1227--1244,
  2015.

\bibitem[Cesa-Bianchi and Lugosi(2006)]{book/Cambridge/cesa2006prediction}
Nicolo Cesa-Bianchi and G{\'a}bor Lugosi.
\newblock \emph{Prediction, {L}earning, and {G}ames}.
\newblock Cambridge {U}niversity {P}ress, 2006.

\bibitem[Cesa{-}Bianchi et~al.(1997)Cesa{-}Bianchi, Freund, Haussler, Helmbold,
  Schapire, and Warmuth]{JACM'97:doubling-trick}
Nicol{\`{o}} Cesa{-}Bianchi, Yoav Freund, David Haussler, David~P. Helmbold,
  Robert~E. Schapire, and Manfred~K. Warmuth.
\newblock How to use expert advice.
\newblock \emph{Journal of the ACM}, 44\penalty0 (3):\penalty0 427--485, 1997.

\bibitem[Cesa{-}Bianchi et~al.(2005)Cesa{-}Bianchi, Mansour, and
  Stoltz]{COLT'05:second-order-Hedge}
Nicol{\`{o}} Cesa{-}Bianchi, Yishay Mansour, and Gilles Stoltz.
\newblock Improved second-order bounds for prediction with expert advice.
\newblock In \emph{Proceedings of the 18th Annual Conference on Learning Theory
  (COLT)}, pages 217--232, 2005.

\bibitem[Chiang et~al.(2012)Chiang, Yang, Lee, Mahdavi, Lu, Jin, and
  Zhu]{COLT'12:variation-Yang}
Chao-Kai Chiang, Tianbao Yang, Chia-Jung Lee, Mehrdad Mahdavi, Chi-Jen Lu, Rong
  Jin, and Shenghuo Zhu.
\newblock Online optimization with gradual variations.
\newblock In \emph{Proceedings of the 25th Conference On Learning Theory
  (COLT)}, pages 6.1--6.20, 2012.

\bibitem[Freund and Schapire(1997)]{JCSS'97:boosting}
Yoav Freund and Robert~E. Schapire.
\newblock A decision-theoretic generalization of on-line learning and an
  application to boosting.
\newblock \emph{Journal of Computer and System Sciences}, 55\penalty0
  (1):\penalty0 119--139, 1997.

\bibitem[Gaillard et~al.(2014)Gaillard, Stoltz, and van
  Erven]{COLT'14:second-order-Hedge}
Pierre Gaillard, Gilles Stoltz, and Tim van Erven.
\newblock A second-order bound with excess losses.
\newblock In \emph{Proceedings of The 27th Conference on Learning Theory
  (COLT)}, pages 176--196, 2014.

\bibitem[Hazan(2016)]{book'16:Hazan-OCO}
Elad Hazan.
\newblock Introduction to {O}nline {C}onvex {O}ptimization.
\newblock \emph{Foundations and Trends in Optimization}, 2\penalty0
  (3-4):\penalty0 157--325, 2016.

\bibitem[Hazan and Kale(2008)]{COLT'08:Hazan-variation}
Elad Hazan and Satyen Kale.
\newblock Extracting certainty from uncertainty: Regret bounded by variation in
  costs.
\newblock In \emph{Proceedings of the 21st Annual Conference on Learning Theory
  (COLT)}, pages 57--68, 2008.

\bibitem[Hazan et~al.(2007)Hazan, Agarwal, and Kale]{journals/ml/HazanAK07}
Elad Hazan, Amit Agarwal, and Satyen Kale.
\newblock Logarithmic regret algorithms for online convex optimization.
\newblock \emph{Machine Learning}, 69\penalty0 (2-3):\penalty0 169--192, 2007.

\bibitem[Herbster and Warmuth(1998)]{journals/ml/HerbsterW98}
Mark Herbster and Manfred~K. Warmuth.
\newblock Tracking the best expert.
\newblock \emph{Machine Learning}, 32\penalty0 (2):\penalty0 151--178, 1998.

\bibitem[Herbster and Warmuth(2001)]{JMLR'01:Herbster}
Mark Herbster and Manfred~K. Warmuth.
\newblock Tracking the best linear predictor.
\newblock \emph{Journal of Machine Learning Research}, 1:\penalty0 281--309,
  2001.

\bibitem[Jadbabaie et~al.(2015)Jadbabaie, Rakhlin, Shahrampour, and
  Sridharan]{AISTATS'15:dynamic-optimistic}
Ali Jadbabaie, Alexander Rakhlin, Shahin Shahrampour, and Karthik Sridharan.
\newblock Online optimization : Competing with dynamic comparators.
\newblock In \emph{Proceedings of the 18th International Conference on
  Artificial Intelligence and Statistics (AISTATS)}, pages 398--406, 2015.

\bibitem[Littlestone and Warmuth(1994)]{journals/iandc/LittlestoneW94}
Nick Littlestone and Manfred~K. Warmuth.
\newblock The weighted majority algorithm.
\newblock \emph{Information and Computation}, 108\penalty0 (2):\penalty0
  212--261, 1994.

\bibitem[Luo and Schapire(2015)]{COLT'15:Luo-AdaNormalHedge}
Haipeng Luo and Robert~E. Schapire.
\newblock Achieving all with no parameters: {AdaNormalHedge}.
\newblock In \emph{Proceedings of the 28th Annual Conference Computational
  Learning Theory (COLT)}, pages 1286--1304, 2015.

\bibitem[Mokhtari et~al.(2016)Mokhtari, Shahrampour, Jadbabaie, and
  Ribeiro]{CDC'16:dynamic-sc}
Aryan Mokhtari, Shahin Shahrampour, Ali Jadbabaie, and Alejandro Ribeiro.
\newblock Online optimization in dynamic environments: Improved regret rates
  for strongly convex problems.
\newblock In \emph{Proceedings of the 55th IEEE Conference on Decision and
  Control (CDC)}, pages 7195--7201, 2016.

\bibitem[Rakhlin and Sridharan(2013)]{conf/colt/RakhlinS13}
Alexander Rakhlin and Karthik Sridharan.
\newblock Online learning with predictable sequences.
\newblock In \emph{Proceedings of the 26th Conference On Learning Theory
  (COLT)}, pages 993--1019, 2013.

\bibitem[Shalev-Shwartz(2007)]{thesis:shai2007}
Shai Shalev-Shwartz.
\newblock Online {L}earning: {T}heory, {A}lgorithms and {A}pplications.
\newblock \emph{PhD Thesis}, 2007.

\bibitem[Srebro et~al.(2010)Srebro, Sridharan, and Tewari]{NIPS'10:smooth}
Nathan Srebro, Karthik Sridharan, and Ambuj Tewari.
\newblock Smoothness, low noise and fast rates.
\newblock In \emph{Advances in Neural Information Processing Systems 23
  (NIPS)}, pages 2199--2207. 2010.

\bibitem[Sugiyama and Kawanabe(2012)]{book/mit/sugiyama2012machine}
Masashi Sugiyama and Motoaki Kawanabe.
\newblock \emph{Machine {L}earning in {N}on-stationary {E}nvironments:
  Introduction to {C}ovariate {S}hift {A}daptation}.
\newblock The MIT Press, 2012.

\bibitem[Syrgkanis et~al.(2015)Syrgkanis, Agarwal, Luo, and
  Schapire]{NIPS'15:fast-rate-game}
Vasilis Syrgkanis, Alekh Agarwal, Haipeng Luo, and Robert~E. Schapire.
\newblock Fast convergence of regularized learning in games.
\newblock In \emph{Advances in Neural Information Processing Systems 28
  (NIPS)}, pages 2989--2997, 2015.

\bibitem[van Erven and Koolen(2016)]{NIPS'16:MetaGrad}
Tim van Erven and Wouter~M. Koolen.
\newblock Metagrad: Multiple learning rates in online learning.
\newblock In \emph{Advances in Neural Information Processing Systems 29
  (NIPS)}, pages 3666--3674, 2016.

\bibitem[Yang et~al.(2016)Yang, Zhang, Jin, and Yi]{ICML'16:Yang-smooth}
Tianbao Yang, Lijun Zhang, Rong Jin, and Jinfeng Yi.
\newblock Tracking slowly moving clairvoyant: Optimal dynamic regret of online
  learning with true and noisy gradient.
\newblock In \emph{Proceedings of the 33rd International Conference on Machine
  Learning (ICML)}, pages 449--457, 2016.

\bibitem[Yuan and Lamperski(2020)]{AAAI'20:Jianjun}
Jianjun Yuan and Andrew~G. Lamperski.
\newblock Trading-off static and dynamic regret in online least-squares and
  beyond.
\newblock In \emph{Proceedings of the 34th {AAAI} Conference on Artificial
  Intelligence (AAAI)}, pages 6712--6719, 2020.

\bibitem[Zhang et~al.(2017)Zhang, Yang, Yi, Jin, and
  Zhou]{NIPS'17:zhang-dynamic-sc-smooth}
Lijun Zhang, Tianbao Yang, Jinfeng Yi, Rong Jin, and Zhi-Hua Zhou.
\newblock Improved dynamic regret for non-degeneracy functions.
\newblock In \emph{Advances in Neural Information Processing Systems 30
  (NIPS)}, pages 732--741, 2017.

\bibitem[Zhang et~al.(2018{\natexlab{a}})Zhang, Lu, and
  Zhou]{NIPS'18:Zhang-Ader}
Lijun Zhang, Shiyin Lu, and Zhi-Hua Zhou.
\newblock Adaptive online learning in dynamic environments.
\newblock In \emph{Advances in Neural Information Processing Systems 31
  (NeurIPS)}, pages 1330--1340, 2018{\natexlab{a}}.

\bibitem[Zhang et~al.(2018{\natexlab{b}})Zhang, Yang, Jin, and
  Zhou]{ICML'18:zhang-dynamic-adaptive}
Lijun Zhang, Tianbao Yang, Rong Jin, and Zhi-Hua Zhou.
\newblock Dynamic regret of strongly adaptive methods.
\newblock In \emph{Proceedings of the 35th International Conference on Machine
  Learning (ICML)}, pages 5877--5886, 2018{\natexlab{b}}.

\bibitem[Zhao and Zhang(2020)]{arxiv:technique-note}
Peng Zhao and Lijun Zhang.
\newblock Improved analysis for dynamic regret of strongly convex and smooth
  functions.
\newblock \emph{ArXiv preprint}, arXiv:2006.05876, 2020.

\bibitem[Zhao et~al.(2019)Zhao, Wang, Xie, Guo, and Zhou]{TKDE'19:DFOP}
Peng Zhao, Xinqiang Wang, Siyu Xie, Lei Guo, and Zhi-Hua Zhou.
\newblock Distribution-free one-pass learning.
\newblock \emph{IEEE Transaction on Knowledge and Data Engineering}, 2019.
\newblock \doi{10.1109/TKDE.2019.2937078}.

\bibitem[Zhao et~al.(2020)Zhao, Wang, Zhang, and Zhou]{AISTATS'20:BCO}
Peng Zhao, Guanghui Wang, Lijun Zhang, and Zhi-Hua Zhou.
\newblock Bandit convex optimization in non-stationary environments.
\newblock In \emph{Proceedings of the 23rd International Conference on
  Artificial Intelligence and Statistics (AISTATS)}, pages 1508--1518, 2020.

\bibitem[Zinkevich(2003)]{ICML'03:zinkvich}
Martin Zinkevich.
\newblock Online convex programming and generalized infinitesimal gradient
  ascent.
\newblock In \emph{Proceedings of the 20th International Conference on Machine
  Learning (ICML)}, pages 928--936, 2003.

\end{thebibliography}
\bibliographystyle{unsrtnat}

\newpage
\appendices
\section{Proof of Gradient-Variation Bound}
\label{sec:appendix-variation}
In this section we provide proofs of gradient-variation bounds, including analysis of the expert-algorithm and meta-algorithm, as well as the proof of the overall dynamic regret bound.

\subsection{Analysis of Expert-Algorithm (Online Extra-Gradient Descent)}
In this part we analyze the expert-algorithm of Sword$_{\text{var}}$, namely, the online extra-gradient descent. We first restate the gradient-variation static regret proved by~\citet{COLT'12:variation-Yang} as follows.
\begin{myThm}
\label{thm:OEGD-static-regret}
Under Assumptions~\ref{assumption:bounded-gradient},~\ref{assumption:bounded-domain}, and~\ref{assumption:smoothness}, by choosing $\eta \leq \frac{1}{4L}$, for any $\x \in \X$, OEGD~\eqref{alg:OEGD} satisfies
\begin{equation*}
  \sum_{t=1}^T f_t(\x_{t}) - \sum_{t=1}^T  f_t(\x) \leq \frac{2}{\eta} + 2\eta \sum_{t=2}^{T} \sup_{\x\in \X} \norm{\nabla f_{t-1}(\x) - \nabla f_t(\x)}_2^2 + GD = \O\Big(\frac{1}{\eta} + \eta V_T\Big).
\end{equation*}
\end{myThm}
Therefore, by choosing $\eta = \min\{1/(4L), 1/\sqrt{V_T}\}$, OEGD achieves an $\O(\sqrt{V_T})$ static regret. Note that the unpleasant dependence on $V_T$ can be eliminated by the doubling trick~\citep{JACM'97:doubling-trick}, because the gradient variation $V_T$ is empirically evaluable at each iteration.

Recall that the static regret is a special case of the universal dynamic regret by setting comparators as the best decision in hindsight, namely, $\u_1 = \u_2 = \ldots = \u_T = \x^* \in \argmin_{\x \in \X} \sum_{t=1}^{T}f_t(\x)$. It is clear that the dynamic regret bound in Theorem~\ref{thm:OEGD-dynamic-regret} recovers the static regret bound in Theorem~\ref{thm:OEGD-static-regret}. Therefore, it is sufficient for us to prove the dynamic regret bound of Theorem~\ref{thm:OEGD-dynamic-regret}, which is presented as follows.

\begin{proof}[Proof of Theorem~\ref{thm:OEGD-dynamic-regret}]
We first decompose the instantaneous dynamic regret as follows:
\begin{align*}
       {} & f_t(\x_t) - f_t(\u_t) \leq \inner{\nabla f_t(\x_t)}{\x_t - \u_t}\\
  =    {} &  \underbrace{\langle \nabla f_t(\x_t) - \nabla f_{t-1}(\xh_t), \x_t - \xh_{t+1}\rangle}_{\mathtt{term~(a)}} + \underbrace{\langle \nabla f_{t-1}(\xh_t), \x_t - \xh_{t+1}\rangle}_{\mathtt{term~(b)}} + \underbrace{\langle \nabla f_t(\x_t), \xh_{t+1} - \u_t \rangle}_{\mathtt{term~(c)}}.
\end{align*}

The three terms will be bounded individually, where we will use the H\"older's inequality to bound term (a), and use the projection lemma (Lemma~\ref{lemma:bregman-divergence}) to bound terms (b) and (c). 

We first investigate term (a). Indeed, the H\"older's inequality implies that
\begin{equation}
	\label{eq:term-a}
	\begin{split}	
	\mathtt{term~(a)} \leq {} & \norm{\nabla f_t(\x_t) - \nabla f_{t-1}(\xh_t)}_2 \norm{\x_t - \xh_{t+1}}_2 \\
	\leq {} & \frac{\eta}{2}\norm{\nabla f_t(\x_t) - \nabla f_{t-1}(\xh_t)}_2^2  + \frac{1}{2\eta} \norm{\x_t - \xh_{t+1}}_2^2 
	\end{split}
\end{equation}
where we make use of the fact that $ab\leq \frac{a^2}{2\eta} + \frac{\eta b^2}{2}$ holds for any $a,b\geq 0$ and $\eta >0$.

From the update procedure in~\eqref{alg:OEGD} and by employing Lemma~\ref{lemma:bregman-divergence}, we have 
\begin{align*}
\langle \hat{\x}_{t+1} - \u_{t}, \eta \nabla f_{t}(\x_{t}) \rangle & \leq \frac{1}{2} \norm{\u_t - \xh_t}_2^2 - \frac{1}{2} \norm{\u_t - \xh_{t+1}}_2^2 - \frac{1}{2} \norm{\xh_{t+1} - \xh_t}_2^2, \\
\langle \x_t - \xh_{t+1}, \eta \nabla f_{t-1}(\xh_{t}) \rangle & \leq \frac{1}{2} \norm{\xh_t - \xh_{t+1}}_2^2 - \frac{1}{2} \norm{\x_t - \xh_{t+1}}_2^2 - \frac{1}{2} \norm{\x_t - \xh_t}_2^2. 
\end{align*}
A simple rearrangement delivers 
\begin{align}
  \mathtt{term~(b)} \leq {} & \frac{1}{2\eta}\big( \norm{\xh_t - \xh_{t+1}}_2^2 - \norm{\x_t - \xh_{t+1}}_2^2 - \norm{\x_t - \xh_{t}}_2^2 \big) \label{eq:term-b}\\
  \mathtt{term~(c)} \leq {} & \frac{1}{2\eta} \big(\norm{\u_t-\xh_{t}}_2^2 - \norm{\u_t-\xh_{t+1}}_2^2 - \norm{\xh_{t+1}-\xh_{t}}_2^2\big) \label{eq:term-c}
\end{align}

So we can combine all three inequalities~\eqref{eq:term-a},~\eqref{eq:term-b},~\eqref{eq:term-c} and get
\begin{align*}
  f_t(\x_t) - f_t(\u_t) \leq {} & \frac{\eta}{2}\norm{\nabla f_t(\x_t) - \nabla f_{t-1}(\xh_t)}_2^2  + \frac{1}{2\eta} \norm{\x_t - \xh_{t+1}}_2^2 \\
   {} & + \frac{1}{2\eta}\left( \norm{\u_t-\xh_{t}}_2^2 - \norm{\u_t-\xh_{t+1}}_2^2 - \norm{\x_t - \xh_{t+1}}_2^2 - \norm{\x_t - \xh_{t}}_2^2\right),
\end{align*}

Summing the above inequality over all iterations, we can bound the dynamic regret as follows,
\begin{align*}
       {} & \sum_{t=1}^{T} f_t(\x_t) - \sum_{t=1}^T f_t(\u_t)   \\
  \leq {} & f_1(\x_1) - f_1(\u_1) + \frac{\eta}{2} \sum_{t=2}^T \norm{\nabla f_t(\x_t) - \nabla f_{t-1}(\xh_t)}_2^2 \\
  {} & \qquad \qquad \quad \quad ~~ + \frac{1}{2\eta}\sum_{t=2}^T \Big( \norm{\u_t-\xh_{t}}_2^2 - \norm{\u_t-\xh_{t+1}}_2^2 - \norm{\x_t - \xh_{t}}_2^2\Big)\\
  \leq  {} & GD + \underbrace{\frac{\eta}{2} \sum_{t=2}^T \norm{\nabla f_t(\x_t) - \nabla f_{t-1}(\xh_t)}_2^2}_{\mathtt{term~(i)}} \\
  {} & \qquad \qquad \quad \quad ~~ + \underbrace{\frac{1}{2\eta} \sum_{t=2}^T \big(\norm{\u_t-\xh_{t}}_2^2 - \norm{\u_t-\xh_{t+1}}_2^2\big)}_{\mathtt{term~(ii)}} - \frac{1}{2\eta}\sum_{t=2}^T \norm{\x_t - \xh_t}_2^2.
\end{align*}

We exploit smoothness to bound term (i),
\begin{align*}
  \mathtt{term~(i)} = {} & \frac{\eta}{2} \sum_{t=2}^T \norm{\nabla f_t(\x_t) - \nabla f_{t-1}(\xh_t)}_2^2 \\
  \leq {} & \frac{\eta}{2} \sum_{t=2}^T 2\big(\norm{\nabla f_t(\x_t) - \nabla f_{t}(\xh_t)}_2^2 + \norm{\nabla f_{t}(\xh_t) - \nabla f_{t-1}(\xh_t)}_2^2\big) \\
  \overset{\eqref{eqn:f:smooth}}{\leq} {} & \eta \sum_{t=2}^T \Big(L^2 \norm{\x_t - \xh_t}_2^2 + \sup_{\x \in \X} \norm{\nabla f_{t}(\x) - \nabla f_{t-1}(\x)}_2^2\Big)\\
  = {} & \eta L^2 \sum_{t=2}^T \norm{\x_t - \xh_t}_2^2 + \eta V_T.
\end{align*}

It suffices to bound term (ii),
\begin{align*}
  \mathtt{term~(ii)} = {} & \frac{1}{2\eta} \sum_{t=2}^T \big(\norm{\u_t-\xh_{t}}_2^2 - \norm{\u_t-\xh_{t+1}}_2^2\big)\\
  \leq {} & \frac{1}{2\eta} \norm{\u_1 - \xh_2}_2^2 + \frac{1}{2\eta} \sum_{t=2}^T \big(\norm{\u_t-\xh_{t}}_2^2 - \norm{\u_{t-1}-\xh_{t}}_2^2\big)\\
  \leq {} & \frac{D^2}{2\eta} + \frac{1}{2\eta}\sum_{t=2}^T \norm{\u_t-\xh_{t} + \u_{t-1}-\xh_{t}}_2 \norm{\u_t - \u_{t-1}}_2\\
  \leq {} & \frac{D^2}{2\eta} + \frac{D}{\eta}\sum_{t=2}^T  \norm{\u_t - \u_{t-1}}_2.
\end{align*}

Putting the above inequalities of terms (i) and (ii) together yields,
\begin{equation}
\label{eq:extra-OGD-ineqn-1}
  \begin{split}
  	   {} & \sum_{t=1}^{T} f_t(\x_t) - \sum_{t=1}^Tf_t(\u_t) \\
  \leq {} & GD + (\eta L^2 - \frac{1}{2\eta}) \sum_{t=2}^T \norm{\x_t - \xh_t}_2^2 + \eta V_T + \frac{D^2}{2\eta} + \frac{D}{\eta}\sum_{t=2}^T  \norm{\u_t - \u_{t-1}}_2 \\
  \leq {} & GD + \eta V_T + \frac{D^2}{2\eta} + \frac{DP_T}{\eta} \
  \end{split}
\end{equation}
where the last step makes use of the condition $\eta \leq 1/(4L)$. This completes the proof.
\end{proof}

\subsection{Analysis of Meta-Algorithm (VariationHedge)}
In this part we analyze the meta-algorithm of Sword$_{\text{var}}$, i.e., VariationHedge. We first present a general meta-regret bound of VariationHedge in Theorem~\ref{thm:dynamic-regret-meta}, which holds for any learning rate $\epsilon > 0$. Then, we prove Theorem~\ref{thm:variation-meta-regret} as a consequence by choosing a proper learning rate. Note that in the meta-regret analysis, we will denote the number of candidate step sizes (namely, the number of experts) simply by $N$ instead of $N_1$ when no confusion can arise.

Let us restate the weight update procedure of VariationHedge. From~\eqref{eq:VariationHedge}, we know that VariationHedge updates the weight $\p_{t+1} \in \Delta_{N}$ by
\begin{equation}
  \label{eq:weight-update-VariationHedge}
  p_{t+1,i} = \frac{\exp\left(-\epsilon\big(\sum_{s=1}^{t} \inner{\nabla f_s(\x_s)}{\x_{s,i}} + \inner{\nabla f_t(\bar{\x}_{t+1})}{\x_{t+1,i}}\big)\right)}{\sum_{i=1}^{N} \exp\left(-\epsilon\big(\sum_{s=1}^{t} \inner{\nabla f_s(\x_s)}{\x_{s,i}} + \inner{\nabla f_t(\bar{\x}_{t+1})}{\x_{t+1,i}}\big)\right)},
\end{equation}
for any $i \in [N]$, where the instrumental output $\bar{\x}_{t+1}$ is carefully designed as
\begin{equation}
  \label{eq:instrumental}
  \bar{\x}_{t+1} = \sum_{i=1}^{N} p_{t,i} \x_{t+1,i}.
\end{equation}
The motivation of the design has been illustrated in Remark~\ref{remark:1}. Note that VariationHedge is actually a specialization of OptimisticHedge by setting the linearized surrogate loss $\ell_{t,i}= \inner{\nabla f_t(\x_t)}{\x_{t,i}}$ and optimism $m_{t+1,i} = \inner{\nabla f_t(\bar{\x}_{t+1})}{\x_{t+1,i}}$. Therefore, by Lemma~\ref{lemma:OptimisticHedge} and the setting of the instrumental output $\bar{\x}_{t+1}$, we have the following result regarding its meta-regret.
\begin{myThm}
\label{thm:dynamic-regret-meta}
Under Assumptions~\ref{assumption:bounded-gradient},~\ref{assumption:bounded-domain} and~\ref{assumption:smoothness}, the meta-regret of the VariationHedge satisfies
\begin{equation}
  \label{eq:dynamic-regret-meta}
  \sum_{t=1}^{T} f_t(\x_t) - \sum_{t=1}^{T} f_t(\x_{t,i}) \leq \frac{2 + \ln N}{\epsilon} + 2\epsilon D^2 V_T + \bigg(2\epsilon D^4L^2 - \frac{1}{4\epsilon}\bigg) \sum_{t=2}^T\norm{\p_{t} - \p_{t-1}}_1^2 + \O(1),
\end{equation}
which holds for any expert $i \in [N]$. $\O(1)$ hides the constants without affecting the final regret order.
\end{myThm}

\begin{proof}[{Proof of Theorem~\ref{thm:dynamic-regret-meta}}]
  By convexity, we know that the dynamic regret with respect to the original loss function is bounded by that with respect to the linearized surrogate loss, namely,
  \begin{align*}
       \sum_{t=1}^{T} f_t(\x_t) - \sum_{t=1}^Tf_t(\x_{t,i})\leq  \sum_{t=1}^{T} \inner{\nabla f_t(\x_t)}{\x_t - \x_{t,i}} = \sum_{t=1}^{T} \inner{\p_t}{\ellb_t} - \sum_{t=1}^T\ell_{t,i}.
  \end{align*}
  Since VariationHedge is a variant of OptimisticHedge by assigning the feedback loss of expert $\Ecal_i$ as $\ell_{t,i} = \inner{\nabla f_t(\x_t)}{\x_{t,i}}$ and the optimism as $m_{t+1,i} = \inner{\nabla f_t(\bar{\x}_{t+1})}{\x_{t+1,i}}$, Lemma~\ref{lemma:OptimisticHedge} implies
  \begin{align}
  &\sum_{t=1}^{T} \inner{\p_t}{\ellb_t} - \sum_{t=1}^T \ell_{t,i}\nonumber\\
  \overset{\eqref{eq:regret-optimistic-Hedge}}{\leq} {} & \epsilon \sum_{t=1}^{T} \norm{\ellb_t - \m_t}_\infty^2 + \frac{2 + \ln N}{\epsilon} - \frac{1}{4\epsilon}\sum_{t=2}^{T} \norm{\p_t - \p_{t-1}}_1^2\nonumber\\
  = {} & \epsilon \sum_{t=1}^{T} \left(\max_{i \in [N]} \inner{\nabla f_t(\x_t) - \nabla f_{t-1}(\xb_t)}{\x_{t,i}}\right)^2 + \frac{2 + \ln N}{\epsilon} - \frac{1}{4\epsilon}\sum_{t=2}^{T} \norm{\p_t - \p_{t-1}}_1^2\nonumber\\
  \leq {} & \epsilon D^2 \sum_{t=1}^{T} \norm{\nabla f_t(\x_t) - \nabla f_{t-1}(\xb_t)}_2^2 + \frac{2 + \ln N}{\epsilon} - \frac{1}{4\epsilon}\sum_{t=2}^{T} \norm{\p_t - \p_{t-1}}_1^2\nonumber\\
  \leq {} & 2 \epsilon D^2 \sum_{t=1}^{T} \left(\norm{\nabla f_t(\x_t) - \nabla f_{t-1}(\x_t)}_2^2 + \norm{\nabla f_{t-1}(\x_t) - \nabla f_{t-1}(\xb_t)}_2^2\right) \nonumber \\
  	{} & \qquad \qquad \qquad \qquad \qquad \qquad \qquad \qquad + \frac{2 + \ln N}{\epsilon} - \frac{1}{4\epsilon}\sum_{t=2}^{T} \norm{\p_t - \p_{t-1}}_1^2\nonumber\\
  \leq {} & 2 \epsilon D^2 \sum_{t=1}^{T}\sup_{\x \in \X}\norm{\nabla f_t(\x) - \nabla f_{t-1}(\x)}_2^2 + 2 \epsilon D^2 L^2 \sum_{t=1}^{T}\norm{\x_t - \xb_t}_2^2 \nonumber \\
  	{} & \qquad \qquad \qquad \qquad \qquad \qquad \qquad \qquad + \frac{2 + \ln N}{\epsilon} - \frac{1}{4\epsilon}\sum_{t=2}^{T} \norm{\p_t - \p_{t-1}}_1^2\label{eq:meta-regret-vartion-1}\\
  \leq {} & 2 \epsilon D^2 V_T + 2 \epsilon D^2 L^2 \sum_{t=1}^{T}\norm{\x_t - \xb_t}_2^2 + \frac{2 + \ln N}{\epsilon} - \frac{1}{4\epsilon}\sum_{t=2}^{T} \norm{\p_t - \p_{t-1}}_1^2 + \O(1)\label{eq:meta-regret-variation-step1},
  \end{align}
  where the second inequality follows from H\"older's inequality inequality and Assumption~\ref{assumption:bounded-domain} (boundedness of domain), and~\eqref{eq:meta-regret-vartion-1} holds due to the smoothness. Notice that the extra $\O(1)$ term appears in~\eqref{eq:meta-regret-variation-step1} because the definition of gradient variation $V_T$ begins from the index of $2$. We will keep the notation of $\O(1)$ without presenting detailed values, as the constant will not affect the regret order. 

  We now focus on the last two terms. Indeed,
\begin{align}
    {} & 2 \epsilon D^2 L^2 \sum_{t=1}^{T}\norm{\x_t - \xb_t}_2^2 - \frac{1}{4\epsilon}\sum_{t=2}^{T} \norm{\p_t - \p_{t-1}}_1^2 \nonumber\\
  = {} & 2 \epsilon D^2 L^2 \sum_{t=1}^{T}\left\Vert \sum_{i=1}^{N}(p_{t,i} - p_{t-1,i})\x_{t,i}\right\Vert_2^2 - \frac{1}{4\epsilon}\sum_{t=2}^{T} \norm{\p_t - \p_{t-1}}_1^2\nonumber\\
  \leq {} & 2 \epsilon D^2 L^2 \sum_{t=1}^{T}\left(\sum_{i=1}^{N}\abs{p_{t,i} - p_{t-1,i}}\norm{\x_{t,i}}_2\right)^2 - \frac{1}{4\epsilon}\sum_{t=2}^{T} \norm{\p_t - \p_{t-1}}_1^2\nonumber\\
  \leq {} & \left(2 \epsilon D^4 L^2 - \frac{1}{4\epsilon}\right) \sum_{t=2}^{T} \norm{\p_t - \p_{t-1}}_1^2 + \O(1)\label{eq:meta-regret-variation-step2},
\end{align}
where the first equality holds due to the definition of the instrumental output $\xb_t$, which is carefully designed to convert the adaptivity $D_\infty$ to the desired gradient variation $V_T$. Besides, the last inequality follows from the boundedness assumption. Notice that the extra $\O(1)$ term is introduced due to that the negative term begins from the index of $2$. 

Therefore, we complete the proof by combining~\eqref{eq:meta-regret-variation-step1} and~\eqref{eq:meta-regret-variation-step2}.
\end{proof}

Theorem~\ref{thm:dynamic-regret-meta} presents a general regret bound for the meta-algorithm, VariationHedge. By appropriate learning rate tuning, we can obtain Theorem~\ref{thm:variation-meta-regret}. We now show the proof as follows.

\begin{proof}[{Proof of Theorem~\ref{thm:variation-meta-regret}}]
According to Theorem~\ref{thm:dynamic-regret-meta}, for any expert $i \in[N]$, the meta-regret of VariationHedge satisfies 
\begin{equation*}
  \sum_{t=1}^{T} f_t(\x_t) - \sum_{t=1}^{T} f_t(\x_{t,i}) \leq \frac{2 + \ln N}{\epsilon} + 2\epsilon D^2 V_T + \bigg(2\epsilon D^4L^2 - \frac{1}{4\epsilon}\bigg) \sum_{t=2}^T\norm{\p_{t} - \p_{t-1}}_1^2  + \O(1).
\end{equation*}
Since $\epsilon\leq\sqrt{1/(8D^4L^2)}$, we have $\big(2\epsilon D^4L^2 - 1/(4\epsilon)) \sum_{t=2}^T\norm{\p_{t} - \p_{t-1}}_1^2<0$. Therefore the meta-regret of VariationHedge is bounded by, 
\begin{equation*}
  \sum_{t=1}^{T} f_t(\x_t) - \sum_{t=1}^{T} f_t(\x_{t,i}) \leq \frac{2 + \ln N}{\epsilon} + 2\epsilon D^2 V_T + \O(1).
\end{equation*}

Let $\epsilon_* = \sqrt{\frac{2 + \ln N}{2D^2 V_T}}$ and $\epsilon_0 = \sqrt{\frac{1}{8D^4L^2}}$, we set the learning rate as $\epsilon = \min\{\epsilon_0, \epsilon_*\}$.  We consider the following two cases:
\begin{itemize}
  \item when $\epsilon^* \leq \epsilon_0$, the meta-regret is at most $$\sum_{t=1}^{T} f_t(\x_t) - \sum_{t=1}^{T} f_t(\x_{t,i}) \leq (2 + \ln N)/\epsilon^* + 2 \epsilon^* D^2 V_T = 2\sqrt{2D^2(2+\ln N)V_T}.$$
  \item when $\epsilon^* \geq \epsilon_0$, the meta-regret is bounded by 
  \[
  \sum_{t=1}^{T} f_t(\x_t) - \sum_{t=1}^{T} f_t(\x_{t,i}) \leq \frac{2 + \ln N}{\epsilon_0} + 2 \epsilon_0 D^2 V_T \leq 4\sqrt{2}D^2L (2+\ln N),
  \] 
  where the last inequality makes use of the condition of $\epsilon^* \geq \epsilon_0$.
\end{itemize}
Hence, taking the two cases into account, the meta-regret is bounded by
\begin{align*}
	 \sum_{t=1}^{T} f_t(\x_t) - \sum_{t=1}^{T} f_t(\x_{t,i}) \leq {} & 2D\sqrt{2V_T(2+\ln N)} + 4\sqrt{2}D^2L (2+\ln N) + \O(1) \\
  = {} & \O\left(\sqrt{(\ln N + V_T) \ln N}\right),
\end{align*}
which completes the proof.
\end{proof}

\subsection{Proof of Gradient-Variation Dynamic Regret Bounds (Theorem~\ref{thm:dynamic-var})}
\label{sec:proof-thm-var}
\begin{proof}[Proof of Theorem~\ref{thm:dynamic-var}]
Notice that the dynamic regret can be decomposed into the following two parts
\begin{equation}
  \label{eq:dynamic-regret-decompose}
  \sum_{t=1}^{T} f_t(\x_t) - \sum_{t=1}^{T} f_t(\u_t) = \underbrace{\sum_{t=1}^{T} f_t(\x_t) - \sum_{t=1}^{T} f_t(\x_{t,i})}_{\mathtt{meta\text{-}regret}} + \underbrace{\sum_{t=1}^{T} f_t(\x_{t,i}) - \sum_{t=1}^{T} f_t(\u_t)}_{\expert},
\end{equation}
which holds for any expert index $i \in [N_1]$. In above, $\{\x_t\}_{t=1,\ldots,T}$ denotes the final output sequence, and $\{\x_{t,i}\}_{t=1,\ldots,T}$ is the prediction sequence of expert $\Ecal_i$. The first part is the difference between cumulative loss of final output sequence and that of prediction sequence of expert $\Ecal_i$, which is introduced by the meta-algorithm and thus named as \emph{meta-regret}; the second part is the dynamic regret of expert $\Ecal_i$ and therefore named as \emph{expert-regret}. 

In the following, we upper bound these two terms respectively. 

\paragraph{Upper bound of meta-regret.} Recall that in Sword$_\text{var}$, the final decision $\x_t$ at iteration $t$ is a weighted combination of  predictions returned from the expert-algorithms, and the weight is updated by the meta-algorithm (VariationHedge). Therefore, we can apply Theorem~\ref{thm:variation-meta-regret} to track any expert $i \in [N_1]$ and obtain the upper bound of the meta-regret,
\begin{equation}
  \label{eq:meta-regret-variation}
  \meta = \sum_{t=1}^{T} f_t(\x_t) - \sum_{t=1}^{T} f_t(\x_{t,i}) \leq 2D\sqrt{2V_T(2+\ln N_1)} + 4\sqrt{2}D^2L (2+\ln N_1).
\end{equation}

\paragraph{Upper bound of expert-regret.} To make the bound in~\eqref{eq:dynamic-regret-decompose} tight, we find the expert $k\in[N_1]$ with the smallest expert-regret. In other words, we need to identify the nearly optimal step size. 

Recall that the optimal step size is $\eta^* = \min\{\frac{1}{4L},\sqrt{(D^2 + 2DP_T)/(2V_T)}\}$. Meanwhile, $V_T = \sum_{t=2}^T \sup_{\x\in \X}\norm{\nabla f_t(\x) - \nabla f_{t-1}(\x)}_2^2 \leq 4G^2T$ due to Assumption~\ref{assumption:bounded-gradient} and Assumption~\ref{assumption:bounded-domain}. Consequently, the possible minimal and maximal values of the optimal step size are
\begin{equation}
  \label{eq:possible-step-size}
  \eta_{\min} = \sqrt{\frac{D^2}{8G^2T}}, \quad \eta_{\max} = \frac{1}{4L}.
\end{equation}

By the construction of the candidate step size pool $\H_{\text{var}}$, we know that the step size therein is monotonically increasing with respect to the index, in particular,  
\[
  \eta_1 = \sqrt{\frac{D^2}{8G^2T}} = \eta_{\min}, \text{ and  }\eta_{N_1} \leq \frac{1}{4L} = \eta_{\max}.
\]
Therefore, we confirm that there exists an integer $k\in [N_1]$ such that $\eta_{k} \leq \eta^* \leq \eta_{k+1} = 2\eta_k$. The gap between the cumulative loss of final decisions and that of expert $k$ can be upper bounded as follows,
\begin{align}
  \mathtt{expert\text{-}regret} & = \sum_{t=1}^{T} f_t(\x_{t,k}) - \sum_{t=1}^{T} f_t(\u_t) \nonumber \\
  & \overset{\eqref{eq:extra-OGD-ineqn-1}}{\leq} \frac{D^2 + 2DP_T}{2\eta_k} + \eta_k V_T + GD\nonumber \\
  & \leq \frac{D^2 + 2DP_T}{\eta^*} + \eta^* V_T + GD \label{eq:upper-step-1-variation} \\
  & \leq 3 \sqrt{V_T(D^2 + 2DP_T)} + 6L(D^2 + 2DP_T) + GD \label{eq:upper-step-2-variation}\\
  & \leq 3 \sqrt{2(V_T + 4L^2D^2 + 8L^2DP_T)(D^2 + 2DP_T)} + GD \label{eq:expert-regret-variation}
\end{align}
where~\eqref{eq:upper-step-1-variation} holds due to $\eta_k \leq \eta^* \leq 2\eta_k$, and~\eqref{eq:expert-regret-variation} follows from $\sqrt{a} + \sqrt{b} \leq \sqrt{2(a+b)}$, $\forall a, b>0$. Meanwhile,~\eqref{eq:upper-step-2-variation} holds by noticing that the optimal step size $\eta^*$ is either $\sqrt{(D^2 + 2DP_T)/(2V_T)}$ or $\frac{1}{4L}$, and therefore
\begin{itemize}
  \item when $\eta^* = \sqrt{(D^2 + 2DP_T)/(2V_T)}$, R.H.S of~\eqref{eq:upper-step-1-variation} $ = \frac{3}{2} \sqrt{2V_T(D^2 + 2DP_T)} + GD$. 
  \item when $\eta^* = \frac{1}{4L}$, R.H.S of~\eqref{eq:upper-step-1-variation} $ = 4L(D^2 + 2DP_T) + \frac{1}{4L}V_T + GD \leq 6L(D^2 + 2DP_T) + GD$, where the last inequality holds due to $1/(4L) \leq \sqrt{(D^2 + 2DP_T)/(2V_T)}$ in this case.
\end{itemize} 
We sum over the upper bounds of two conditions and obtain~\eqref{eq:upper-step-2-variation}. 

\paragraph{Upper bound of dynamic regret.} Combining~\eqref{eq:meta-regret-variation} and~\eqref{eq:expert-regret-variation}, we obtain
\begin{align*}
  & \sum_{t=1}^{T} f_t(\x_t) - \sum_{t=1}^{T} f_t(\u_t) \\
  \overset{\eqref{eq:dynamic-regret-decompose}}{=} &\underbrace{\sum_{t=1}^{T} f_t(\x_t) - \sum_{t=1}^{T} f_t(\x_{t,k})}_{\mathtt{meta\text{-}regret}} + \underbrace{\sum_{t=1}^{T} f_t(\x_{t,k}) - \sum_{t=1}^{T} f_t(\u_t)}_{\mathtt{expert\text{-}regret}} \nonumber \\
  \overset{\eqref{eq:meta-regret-variation}~\eqref{eq:expert-regret-variation}}{\leq} & 2D\sqrt{2V_T(2+\ln N_1)} + 4\sqrt{2}D^2L (2+\ln N_1) \\
    {} & \qquad \qquad \qquad \qquad + 3 \sqrt{2(V_T + 4L^2D^2 + 8L^2DP_T)(D^2 + 2DP_T)} + GD\\
  \leq {} & 3 \sqrt{4(V_T + 4L^2D^2 + 8L^2DP_T)(D^2 + 2DP_T) + 4D^2V_T(2+\ln N_1)} \\
  {} & \qquad \qquad \qquad \qquad + 4\sqrt{2}D^2L (2+\ln N_1) + GD\\
  \leq {} & 6 \sqrt{((3+\ln N_1)V_T + 4L^2D^2 + 8L^2DP_T)(D^2 + 2DP_T)} + 4\sqrt{2}D^2L (2+\ln N_1) + GD\\
  = {} & \O\left(\sqrt{(1 + P_T + V_T)(1 + P_T)}\right).
\end{align*} 
The derivation uses the inequality of $\sqrt{a} + \sqrt{b} \leq \sqrt{2(a+b)}$, $\forall a,b\geq 0$. Meanwhile, we treat the double logarithmic factor in $T$ as a constant, following previous studies~\citep{ALT'12:closer-adaptive-regret,COLT'15:Luo-AdaNormalHedge}. We finally remark that the universal dynamic regret presented above holds for any  sequence of feasible comparators.
\end{proof}

\section{Proof of Small-Loss Bound}
\label{sec:appendix-small-loss}
In this section we provide proofs of small-loss bounds, including analysis of the expert-algorithm and meta-algorithm, as well as the proof of overall dynamic regret.

\subsection{Analysis of Expert-Algorithm (Online Gradient Descent)}
In this part we analyze the expert-algorithm of the Sword$_{\text{var}}$ algorithm, namely, the online gradient descent. We will present the proof of the small-loss dynamic regret bound (Theorem~\ref{thm:dynamic-OGD}). Before that, in the following we first restate the small-loss static regret bound~\citep[Theorem 2]{NIPS'10:smooth} as well as its proof.

\begin{myThm}[{Theorem 2 of~\citet{NIPS'10:smooth}}] 
\label{thm:ogd}
Under Assumptions~\ref{assumption:bounded-domain},~\ref{assumption:smoothness}, and~\ref{assumption:non-negative}, by choosing any step size $\eta \leq \frac{1}{4L}$, OGD satisfies
\begin{align*}
  \sum_{t=1}^T f_t(\x_{t}) - \sum_{t=1}^T  f_t(\x^*) \leq \frac{D^2}{2 \eta (1-2 \eta L)}  +  \frac{2 \eta L}{(1-2 \eta L)} \sum_{t=1}^T  f_t(\x^*) = \O \Big( \frac{1}{\eta} + \eta F_T^* \Big)
\end{align*}
for any $\x^* \in \X$. Here $F^*_T = \sum_{t=1}^T  f_t(\x^*)$ is the cumulative loss of the comparator benchmark $\x^*$, which is usually set as the best decision in hindsight, i.e., $\x^* = \argmin_{\x \in \X} \sum_{t=1}^T  f_t(\x)$.
\end{myThm}
Theorem~\ref{thm:ogd} indicates an $\O(\sqrt{F^*_T})$ regret bound with a proper choice of step size, which is tighter than the minimax rate of $\O(\sqrt{T})$ when the cumulative loss is small.

\begin{proof}[Proof of Theorem~\ref{thm:ogd}]
First, notice that Assumptions~\ref{assumption:non-negative} and \ref{assumption:smoothness} imply $f_t(\cdot)$ is non-negative and $L$-smooth. From the self-bounding property of smooth functions~\citep{NIPS'10:smooth}, as shown in Lemma~\ref{lem:smooth}, we have
\begin{equation} \label{eqn:smooth:key}
\|\nabla f_t(\x)\|_2^2 \leq 4 L  f_t (\x), \ \forall \x \in \X.
\end{equation}

Define $\x_{t+1}'=\x_t - \eta \nabla f_t(\x_t)$. For any $\x \in \X$, we have
\begin{equation} \label{eqn:ogd:1:old}
\begin{split}
 f_t(\x_{t}) - f_t(\x)\leq {} & \langle \nabla f_t(\x_{t}), \x_{t} - \x\rangle = \frac{1}{\eta} \langle \x_{t}  - \x_{t+1}', \x_{t} - \x\rangle \\
= {} & \frac{1}{2 \eta} \left( \|\x_t-\x\|_2^2 - \|\x_{t+1}'-\x\|_2^2 + \|\x_t  - \x_{t+1}'\|_2^2 \right) \\
= {} & \frac{1}{2 \eta} \left( \|\x_{t}-\x\|_2^2 - \|\x_{t+1}'-\x\|_2^2 \right) + \frac{\eta}{2 } \|\nabla f_t(\x_{t})\|_2^2  \\
\overset{\eqref{eqn:smooth:key}}{\leq} {} & \frac{1}{2 \eta} \left( \|\x_{t}-\x\|_2^2 - \|\x_{t+1}-\x\|_2^2 \right) + 2 \eta L f_t(\x_t)\\
\end{split}
\end{equation}
Summing the above inequality over all iterations, we have
\[
\sum_{t=1}^T f_t(\x_{t}) - \sum_{t=1}^T  f_t(\x) \leq \frac{1}{2 \eta} \|\x_{1}-\x\|_2^2 + 2 \eta L \sum_{t=1}^T f_t(\x_t) \leq \frac{D^2}{2 \eta}  + 2 \eta L \sum_{t=1}^T f_t(\x_t)
\]
which implies
\[
(1-2 \eta L)  \left( \sum_{t=1}^T f_t(\x_{t}) - \sum_{t=1}^T  f_t(\x)  \right) \leq \frac{D^2}{2 \eta} + 2 \eta L \sum_{t=1}^T  f_t(\x).
\]
We complete the proof by dividing both sides by $(1-2 \eta L)$, as the step size satisfies $\eta \leq 1/(4L)$.
\end{proof}

\begin{proof}[Proof of Theorem~\ref{thm:dynamic-OGD}]
Let $\x_{t+1}'=\x_t - \eta \nabla f_t(\x_t)$. Following the standard analysis, we have
\begin{equation} \label{eqn:ogd:1}
\begin{split}
 f_t(\x_{t}) - f_t(\u_t) \leq {} & \langle \nabla f_t(\x_{t}), \x_{t} - \u_t\rangle = \frac{1}{\eta} \langle \x_{t}  - \x_{t+1}', \x_{t} - \u_t\rangle \\
= {} & \frac{1}{2 \eta} \left( \|\x_t-\u_t\|_2^2 - \|\x_{t+1}'-\u_t\|_2^2 + \|\x_t  - \x_{t+1}'\|_2^2 \right) \\
= {} & \frac{1}{2 \eta} \left( \|\x_{t}-\u_t\|_2^2 - \|\x_{t+1}'-\u_t\|_2^2 \right) + \frac{\eta}{2 } \|\nabla f_t(\x_{t})\|_2^2  \\
\overset{\text{(\ref{eqn:smooth:key})}}{\leq} {} & \frac{1}{2 \eta} \left( \norm{\x_{t}-\u_t}_2^2 - \norm{\x_{t+1}-\u_t}_2^2 \right) + 2 \eta L f_t(\x_t)\\
\end{split}
\end{equation}
Summing the above inequality over all iterations, we have
\begin{align} 
 {} & \sum_{t=1}^T f_t(\x_{t}) - \sum_{t=1}^T  f_t(\u_t) \nonumber\\
 \leq {} & \frac{1}{2\eta} \sum_{t=1}^{T} \left( \norm{\x_{t}-\u_t}_2^2 - \norm{\x_{t+1}-\u_t}_2^2 \right) + 2\eta L \sum_{t=1}^T f_t(\x_t) \nonumber\\ 
 \leq {} & \frac{D^2}{2\eta} + \frac{1}{2\eta}\sum_{t=2}^{T} \left( \norm{\x_{t}-\u_t}_2^2 - \norm{\x_{t}-\u_{t-1}}_2^2 \right) + 2\eta L \sum_{t=1}^T f_t(\x_t) \nonumber\\ 
  \leq {} & \frac{D^2}{2\eta} + \frac{D}{\eta}\sum_{t=2}^{T}\norm{\u_t - \u_{t-1}}_2 + 2\eta L \sum_{t=1}^T f_t(\x_t). \label{eqn:ogd:2}
\end{align}
We complete the proof by simplifying~\eqref{eqn:ogd:2}.
\end{proof}

\subsection{Analysis of Meta-Algorithm (vanilla Hedge)}
\label{sec:SMALL-meta-analysis}
In this part we analyze the meta-algorithm of Sword$_{\text{small}}$, i.e., the vanilla Hedge with linearized surrogate loss. Notice that the vanilla Hedge can be treated as a special case of OptimisticHedge by setting $\ell_{t,i} = \inner{\nabla f_t(\x_t)}{\x_{t,i}}$ and $m_{t,i} = 0$ for all $i\in[N_2]$. Therefore, we will prove the following meta-regret bound based on Lemma~\ref{lemma:OptimisticHedge} and the smoothness of the loss function. Note that in the meta-regret analysis, we will denote the number of candidate step sizes (namely, the number of experts) simply by $N$ instead of $N_2$ when no confusion can arise.

\begin{myThm}
\label{thm:small-loss-meta-regret}
Under Assumptions~\ref{assumption:bounded-domain},~\ref{assumption:smoothness} and~\ref{assumption:non-negative}, by setting the learning rate optimally as $\epsilon=\sqrt{(2+\ln N)/(D^2 \bar{F}_T)}$, the meta-regret of the vanilla Hedge satisfies,
\begin{equation}
  \label{eq:dynamic-regret-meta-small-loss}
  \sum_{t=1}^{T} f_t(\x_t) - \sum_{t=1}^Tf_t(\x_{t,i})\leq 4 LD^2(2+\ln N)+\sqrt{L(2+\ln N)F_T^i},
\end{equation}
where $\bar{F}_T = \sum_{t=1}^T\norm{\nabla f_t(\x_t)}_2^2$ is the cumulative gradient norm, and $F_T^i = \sum_{t=1}^{T} f_t(\x_{t,i})$ is the cumulative loss of expert $\Ecal_i$. The result holds for any $i \in [N]$. 
\end{myThm}

\begin{proof}[{Proof of Theorem~\ref{thm:small-loss-meta-regret}}]
  Similar to the argument in the proof of Theorem~\ref{thm:dynamic-regret-meta}, the dynamic regret with respect to the original loss is bounded by that with respect to the surrogate loss, namely,
  \begin{align*}
       \sum_{t=1}^{T} f_t(\x_t) -\sum_{t=1}^T f_t(\x_{t,i})\leq  \sum_{t=1}^{T} \inner{\nabla f_t(\x_t)}{\x_t - \x_{t,i}} = \sum_{t=1}^{T} \inner{\p_t}{\ellb_t} - \sum_{t=1}^T\ell_{t,i},
  \end{align*}
  where $\ell_{t,i} = \inner{\nabla f_t(\x_t)}{\x_{t,i}}$ and $\p_t$ is updated by~\eqref{eq:OptimisticHedge} by setting $m_{t+1,i}=0$. Since the vanilla Hedge with surrogate loss function can be seen as a special OptimisticHedge, Lemma~\ref{lemma:OptimisticHedge} implies
  \begin{align}
  &\sum_{t=1}^{T} \inner{\p_t}{\ellb_t} - \sum_{t=1}^T \ell_{t,i}\nonumber\\
  \overset{\eqref{eq:regret-optimistic-Hedge}}{\leq} {} & \frac{2 + \ln N}{\epsilon} + \epsilon \sum_{t=1}^{T} \norm{\ellb_t - \m_t}_\infty^2 - \frac{1}{4\epsilon}\sum_{t=2}^{T} \norm{\p_t - \p_{t-1}}_1^2\nonumber\\
  = {} & \frac{2 + \ln N}{\epsilon} + \epsilon \sum_{t=1}^{T} \left(\max_{i \in [N]} \inner{\nabla f_t(\x_t)}{\x_{t,i}}\right)^2 - \frac{1}{4\epsilon}\sum_{t=2}^{T} \norm{\p_t - \p_{t-1}}_1^2\nonumber\\
  \leq {} & \frac{2 + \ln N}{\epsilon} + \epsilon D^2\sum_{t=1}^{T} \norm{\nabla f_t(\x_t)}_2^2 \nonumber 
  \end{align}
  where the last inequality makes use of the Jensen's inequality and drops the negative term. 

  By setting the learning rate as $\epsilon=\sqrt{(2+\ln N)/(D^2 \bar{F}_T)}$, the meta-regret is upper bounded by
  \begin{align*}
  \sum_{t=1}^{T} f_t(\x_t) - \sum_{t=1}^Tf_t(\x_{t,i})\leq 2D\sqrt{(2+\ln N)\sum_{t=1}^T\norm{\nabla f_t(\x_t)}_2^2} \leq 4D\sqrt{L(2+\ln N)\sum_{t=1}^Tf_t(\x_t)}.
  \end{align*}
  The last inequality makes use of the self-bounding property of non-negative and smooth functions (Lemma~\ref{lem:smooth}), which states that for any non-negative $L$-smooth functions $f$, we have $\norm{\nabla f_t(\x)}_2^2\leq 4Lf_t(\x)$. We mention that the optimal learning rate tuning depends on the unknown cumulative gradient norm $\bar{F}_T = \sum_{t=1}^T\norm{\nabla f_t(\x_t)}_2^2$, and issue can be easily addressed by the doubling trick~\citep{JACM'97:doubling-trick} or the self-confident tuning~\citep{JCSS'02:Auer-self-confident}.

  Furthermore, the right hand side is the cumulative loss of decisions returned by the meta-algorithm, which can be further converted to the cumulative loss of decisions returned by expert $\Ecal_i$. The conversion can be achieved by applying Lemma~\ref{lemma:inquality-shai}, which shows that $x-y\leq\sqrt{ax}$ implies $x-y\leq a+ \sqrt{ay},$ for any $x,y,a\in\mathbb{R}^+$. Since all loss functions are non-negative, we have
  \[
    \sum_{t=1}^{T} f_t(\x_t) - \sum_{t=1}^T f_t(\x_{t,i})\leq 16 LD^2(2+\ln N)+4D\sqrt{L(2+\ln N)F_T^i},
  \]
  which completes the proof.
\end{proof}

\subsection{Proof of Small-Loss Dynamic Regret Bounds (Theorem~\ref{thm:dynamic-small})}
\label{sec:proof-thm-small-loss}
\begin{proof}[Proof of Theorem~\ref{thm:dynamic-small}]
The proof is analogous to that of Theorem~\ref{thm:dynamic-var}, where the the dynamic regret is decomposed into the following two parts
\begin{equation}
  \label{eq:dynamic-regret-decompose-small-loss}
  \sum_{t=1}^{T} f_t(\x_t) - \sum_{t=1}^{T} f_t(\u_t) = \underbrace{\sum_{t=1}^{T} f_t(\x_t) - \sum_{t=1}^{T} f_t(\x_{t,i})}_{\mathtt{meta\text{-}regret}} + \underbrace{\sum_{t=1}^{T} f_t(\x_{t,i}) - \sum_{t=1}^{T} f_t(\u_t)}_{\expert},
\end{equation}

\paragraph{Upper bound of meta-regret.} According to Theorem~\ref{thm:small-loss-meta-regret}, we know that for any expert index $i \in [N_2]$ the meta-regret of Sword$_{\text{small}}$ is bounded by 
\begin{equation}
  \label{eq:expert-advice-small-loss-step1}
  \meta = \sum_{t=1}^{T} f_t(\x_t) - \sum_{t=1}^Tf_t(\x_{t,i})\leq 16 LD^2(2+\ln N_2)+4D\sqrt{L(2+\ln N_2)F_T^i},
\end{equation}
where $F_T^i = \sum_{t=1}^Tf_t(\x_{t,i})$ is the cumulative loss of expert $\Ecal_i$.

\paragraph{Upper bound of expert-regret.} Similar to the argument in Section~\ref{sec:proof-thm-var}, we identify that the optimal step size is $\eta^* = \min\{1/(4L),\sqrt{(D^2 + 2DP_T)/(16LF_T)}\}$. Meanwhile, $F_T = \sum_{t=1}^T f_t(\u_t) \leq GDT$ due to Assumption~\ref{assumption:bounded-gradient} and Assumption~\ref{assumption:bounded-domain}. As a result, the possible minimal and maximal values of the optimal step size are
\begin{equation}
  \label{eq:possible-step-size}
  \eta_{\min} = \sqrt{\frac{D}{16LGT}}, \quad \eta_{\max} = \frac{1}{4L}.
\end{equation}
By the construction of the candidate step size pool $\H_{\text{small}}$ in~\eqref{eq:step-size-pool-small-loss}, we know that the step size therein is monotonically increasing with respect to the index, and $\eta_1 = \sqrt{\frac{D}{16LGT}} = \eta_{\min}$, $\eta_{N_2} \leq \frac{1}{4L} = \eta_{\max}$. Therefore, we confirm that there exists an integer $k\in [N_2]$ such that $\eta_{k} \leq \eta^* \leq \eta_{k+1} = 2\eta_{k}$. 

We proceed to upper bound the expert-regret for the expert $k$ as follows.
\begin{align}
  \mathtt{expert\text{-}regret} = {} & \sum_{t=1}^{T} f_t(\x_{t,k}) - \sum_{t=1}^{T} f_t(\u_t) \nonumber \\
  \overset{\eqref{eqn:ogd:2}}{\leq} {} & \frac{D^2 + 2DP_T}{2\eta_k(1-2\eta_k L)} + \frac{2\eta_k L}{1-2\eta_k L} F_T\nonumber \\
  \leq {} & \frac{D^2 + 2DP_T}{2\eta_k} + 4\eta_k L F_T\label{eq:upper-step-0} \\
  \leq {} & \frac{2(D^2 + 2DP_T)}{\eta^*} + 4\eta^* L F_T \label{eq:upper-step-1} \\
  \leq {} & 12\sqrt{LF_T(D^2 + 2DP_T)} + 9L(D^2 + 2DP_T) \label{eq:upper-step-2}\\
  \leq {} & 12\sqrt{2L(F_T + D^2 + 2DP_T)(D^2 + 2DP_T)}\label{eq:expert-regret-small-loss}
\end{align}
where~\eqref{eq:upper-step-0} uses the fact that $\eta_k\leq\frac{1}{4L}$,~\eqref{eq:upper-step-1} holds due to $\eta_k \leq \eta^* \leq 2\eta_k$, and~\eqref{eq:expert-regret-small-loss} follows because of $\sqrt{a} + \sqrt{b} \leq \sqrt{2(a+b)}$, $\forall a, b>0$. Meanwhile,~\eqref{eq:upper-step-2} holds by noticing that the optimal step size $\eta^*$ is either $1/(4L)$ or $\sqrt{(D^2 + 2DP_T)/(8LF_T)}$, and therefore
\begin{itemize}
  \item when $\eta^* = \sqrt{(D^2 + 2DP_T)/(16LF_T)}$, R.H.S of~\eqref{eq:upper-step-1} $ = 12 \sqrt{LF_T(D^2 + 2DP_T)}$. 
  \item when $\eta^* = 1/(4L)$, R.H.S of~\eqref{eq:upper-step-1} $ = 8L(D^2 + 2DP_T) + F_T \leq 9L(D^2 + 2DP_T)$, where the last inequality holds due to $1/(4L) \leq \sqrt{(D^2 + 2DP_T)/(16LF_T)}$ in this case.
\end{itemize} 
We sum over the upper bounds of two conditions and obtain~\eqref{eq:upper-step-2}. 

\paragraph{Upper bound of dynamic regret.} Combining~\eqref{eq:expert-advice-small-loss-step1} and~\eqref{eq:expert-regret-small-loss}, we get
\begin{align*}
  & \sum_{t=1}^{T} f_t(\x_t) - \sum_{t=1}^{T} f_t(\u_t) \\
  \overset{\eqref{eq:dynamic-regret-decompose-small-loss}}{=} &\underbrace{\sum_{t=1}^{T} f_t(\x_t) - \sum_{t=1}^{T} f_t(\x_{t,k})}_{\meta} + \underbrace{\sum_{t=1}^{T} f_t(\x_{t,k}) - \sum_{t=1}^{T} f_t(\u_t)}_{\expert} \nonumber \\
  \overset{\eqref{eq:expert-advice-small-loss-step1}~\eqref{eq:expert-regret-small-loss}}{\leq} & 16 LD^2(2+\ln N_2)+ 4D\sqrt{L(2+\ln N_2)F_T^k} + 12\sqrt{2L(F_T + D^2 + 2DP_T)(D^2 + 2DP_T)}\\
  \overset{\eqref{eq:expert-regret-small-loss}}{\leq} & 16 LD^2(2+\ln N_2) + 4D\sqrt{L(2+\ln N_2)(F_T + 12\sqrt{2L(F_T + D^2 + 2DP_T)(D^2 + 2DP_T)})} \\
  & \qquad \qquad  \qquad  \qquad \qquad \qquad + 12\sqrt{2L(F_T + D^2 + 2DP_T)(D^2 + 2DP_T)}\\
  \leq {}& \O\left(\sqrt{(1 + P_T + F_T)(1 + P_T)}\right),
\end{align*} 
where we frequently make use of the inequality $\sqrt{a} + \sqrt{b} \leq \sqrt{2(a+b)}$, $\forall a,b\geq 0$. Meanwhile, double logarithmic factors in $T$ are treated as a constant, following previous studies~\citep{ALT'12:closer-adaptive-regret,COLT'15:Luo-AdaNormalHedge}. We finally note that the obtained universal dynamic regret holds for any feasible comparator sequence.
\end{proof}

\section{Proof of Best-of-Both-Worlds Bounds}
\label{sec:appendix-bobw}
In this section we provide the regret analysis of the best-of-both-worlds bounds. Specifically, we prove the meta-regret (Theorem~\ref{thm:BEST-meta-regret}) and overall dynamic regret (Theorem~\ref{thm:dynamic-best}).

\begin{proof}[Proof of Theorem~\ref{thm:BEST-meta-regret}]
Since the meta-algorithm used in Sword$_{\text{best}}$ is a specific configuration of OptimisticHedge, we can apply Lemma~\ref{lemma:OptimisticHedge} to upper bound the meta-regret by
\begin{align}
   {} & \sum_{t=1}^T f_t(\x_t) - \sum_{t=1}^T f_t(\x_{t,i}) \nonumber \\
   \overset{\eqref{eq:regret-optimistic-Hedge}}{\leq} {} & \frac{2 + \ln N}{\epsilon} + \epsilon \sum_{t=1}^{T} \norm{\ellb_t - \m_{t}}_{\infty}^2 - \frac{1}{4\epsilon}\sum_{t=2}^{T} \norm{\p_t - \p_{t-1}}_1^2 \nonumber \\
   \overset{\eqref{eq:BEST-OptimisticHedge},\eqref{eq:BEST-optimism}}{=} {} & \frac{2 + \ln N}{\epsilon} + \epsilon \sum_{t=1}^{T} \max_{i\in[N]} \left(\inner{\nabla f_t(\x_t) - M_t}{\x_{t,i}}\right)^2 - \frac{1}{4\epsilon}\sum_{t=2}^{T} \norm{\p_t - \p_{t-1}}_1^2 \nonumber \\
   \leq {} & \frac{2 + \ln N}{\epsilon} + \epsilon D^2 \sum_{t=1}^{T} \norm{\nabla f_t(\x_t) - M_t}_2^2 - \frac{1}{4\epsilon}\sum_{t=2}^{T} \norm{\p_t - \p_{t-1}}_1^2. \label{eq:BEST-meta-bound-1}
\end{align}

On the other hand, noticing that the online function $d_t: \R^d \mapsto \R$ is $2$-strongly convex, so we can employ the regret guarantee of Hedge for strongly convex functions~\citep[Proposition 3.1]{book/Cambridge/cesa2006prediction} and obtain
\[
  \sum_{t=1}^{T} d_t(M_t) \leq \min \left\{\sum_{t=1}^{T} d_t(M^{v}_t), \sum_{t=1}^{T} d_t(M^{s}_t)\right\} + \frac{\ln 2}{2},
\]
or
\begin{equation}
\label{eq:BEST-meta-bound-2}
\sum_{t=1}^{T} \norm{\nabla f_t(\x_t) - M_t}_2^2 \leq \min \left\{\sum_{t=1}^{T} \norm{\nabla f_t(\x_t) - \nabla f_{t-1}(\xb_{t})}_2^2, \sum_{t=1}^{T} \norm{\nabla f_t(\x_t)}_2^2\right\} + \frac{\ln 2}{2}.
\end{equation}
Combining~\eqref{eq:BEST-meta-bound-1} and~\eqref{eq:BEST-meta-bound-2}, we immediately achieve that   
\begin{align*}
  \meta \leq {} & \frac{2 + \ln N}{\epsilon} + \epsilon D^2\Big( \min\{\bar{V}_T, \bar{F}_T\} + \frac{\ln 2}{2}\Big) - \frac{1}{4\epsilon}\sum_{t=2}^{T} \norm{\p_t - \p_{t-1}}_1^2 = \min\{A_T, B_T\}
\end{align*}
where $\bar{V}_T = \sum_{t=2}^{T} \norm{\nabla f_t(\x_t) - \nabla f_{t-1}(\xb_{t})}_2^2$ and $\bar{F}_T = \sum_{t=1}^{T} \norm{\nabla f_t(\x_t)}_2^2$. Besides, $A_T$ and $B_T$ are defined as: 
\begin{align*}
  A_T = \frac{2 + \ln N}{\epsilon} + \epsilon D^2\Big( \bar{V}_T + \frac{\ln 2}{2}\Big) - \frac{1}{4\epsilon}\sum_{t=2}^{T} \norm{\p_t - \p_{t-1}}_1^2, \\
  B_T = \frac{2 + \ln N}{\epsilon} + \epsilon D^2\Big( \bar{F}_T + \frac{\ln 2}{2}\Big) - \frac{1}{4\epsilon}\sum_{t=2}^{T} \norm{\p_t - \p_{t-1}}_1^2.
\end{align*}
Notice that the above terms are essentially the meta-regret of gradient-variation and small-loss bounds, up to constant factors. Therefore, we can make use of their meta-regret analysis to bound the meta-regret of Sword$_\text{best}$. Specifically, by applying the analysis of Theorem~\ref{thm:dynamic-regret-meta}, we know that 
\begin{align*}
  A_T \leq \frac{2 + \ln N}{\epsilon} + \epsilon D^2 \left(2 V_T + \frac{\ln2}{2}\right)
\end{align*}
holds if the learning rate satisfies $\epsilon \leq \sqrt{1/(8D^4L^2)}$. Under such circumstances, the meta-regret can be further bounded by
\begin{align*}
  \meta \leq \min\{A_T, B_T\} \leq \frac{2 + \ln N}{\epsilon} + \epsilon D^2 \left(\min\{2V_T, \bar{F}_T\} + \ln2\right).
\end{align*}

Therefore, we set the learning rate as $\epsilon = \min\{\epsilon_0, \epsilon_*\}$, where 
\[
    \epsilon_0 =\sqrt{1/(8D^4L^2)}, \text{  and } \epsilon_* = \sqrt{(2 + \ln N)/(D^2\min\{2V_T,\bar{F}_T\} + D^2\ln 2)}.
\]
We bound the meta-regret by considering two cases.
\begin{itemize}
  \item When $\epsilon^* \leq \epsilon_0$, the meta-regret is bounded by 
  \begin{align*}
  \meta \leq {} & (2 + \ln N)/\epsilon^* + \epsilon^* D^2 \big(\min\{2V_T, \bar{F}_T\} + \ln2\big) \\
   = {} & 2D\sqrt{(2+\ln N)(\min\{2V_T,\bar{F}_T\} + \ln 2)}.
  \end{align*}
  \item When $\epsilon^* \geq \epsilon_0$, the meta-regret is bounded by 
  \[
  \meta \leq \frac{2 + \ln N}{\epsilon_0} + 2 \epsilon_0 D^2 V_T \leq 4\sqrt{2}D^2L (2+\ln N),
  \] 
  where the last inequality makes use of the condition of $\epsilon^* \geq \epsilon_0$.
\end{itemize}
Hence, taking the two cases into account, the meta-regret is bounded by
\begin{align*}
  \meta \leq {} & 2D\sqrt{(2+\ln N)(\min\{2V_T,\bar{F}_T\} + \ln 2)} + 4\sqrt{2}D^2L (2+\ln N) \\
    = {} & \O\left(\sqrt{(1 + \ln N + \min\{V_T, \bar{F}_T\}) \ln N}\right),
\end{align*}
which completes the proof.
\end{proof}

\begin{proof}[Proof of Theorem~\ref{thm:dynamic-best}]
Notice that the dynamic regret can be decomposed into the following two parts
\begin{equation*}
  \sum_{t=1}^{T} f_t(\x_t) - \sum_{t=1}^{T} f_t(\u_t) = \underbrace{\sum_{t=1}^{T} f_t(\x_t) - \sum_{t=1}^{T} f_t(\x_{t,i})}_{\meta} + \underbrace{\sum_{t=1}^{T} f_t(\x_{t,i}) - \sum_{t=1}^{T} f_t(\u_t)}_{\expert},
\end{equation*}

Since the Sword$_{\text{best}}$ algorithm maintains $N_1 + N_2$ experts, where the first $N_1$ experts run OEGD and the other $N_2$ experts perform OGD. Therefore, the expert-regret can be upper bounded by the minimum of the expert-regret of variation and small-loss algorithms. Meanwhile, in Theorem~\ref{thm:BEST-meta-regret}, we have proved that the meta-regret of Sword$_{\text{best}}$ also achieves a minimum of the meta-regret of variation and small-loss algorithms. Combining the expert-regret and meta-regret analysis, we thus confirm that Sword$_{\text{best}}$ attains a best-of-both-worlds dynamic regret bound.
\end{proof}

\section{Proof of Lemma~\ref{lemma:OptimisticHedge}}
\label{sec:appendix-optimistic}

Lemma~\ref{lemma:OptimisticHedge} guarantees the regret bound of OptimisticHedge, which is originally proved by~\citet{NIPS'15:fast-rate-game} (cf. Theorem 19 of their paper). For self-containedness, we present their proof and adapt to our notations. Before showing the proof, we need to introduce two related lemmas.

The first one is on the property of strongly convex functions.
\begin{myLemma}
\label{lem:strongly-convex}
If $F:\mathcal{X}\mapsto\mathbb{R}$ is a $\lambda$-strongly convex function with respect to a norm $\Vert\cdot\Vert$ and $\x_* = \argmin_{\x\in\X} F(\x)$, then for any $\x \in \X$, we have
\begin{equation}
\label{eq:strongly-convex}
F(\x)\geq F(\x_*)+\frac{\lambda}{2}\norm{\x-\x_*}^2.
\end{equation}
\end{myLemma}
\begin{proof}
According to the definition of strongly convex function, we have $F(\x)\geq F(\x_*)+\inner{\nabla F(\x_*)}{\x-\x_*}+\frac{\lambda}{2}\norm{\x-\x_*}^2$. Besides, by the first order condition of convex functions, we have $\inner{\nabla F(\x_*)}{\x-\x_*} \geq 0$. We complete the proof by combining these two inequalities.
\end{proof}

The second lemma is due to~\citet{NIPS'15:fast-rate-game}, who investigate the \emph{stability} of the Follow the Regularized Leader (FTRL) algorithm. The FTRL algorithm updates the decision $\x_t$ in the form of 
\[
  \x_t = \argmin_{\x\in\mathcal{X}}~\epsilon\inner{\bm{L}}{\x}+\mathcal{R}(\x),
\]
where the regularizer $\mathcal{R}:\X\mapsto\mathbb{R}$ is strongly convex.
\begin{myLemma}
\label{lem:stability-FTRL}
If $\x_* = \argmin_{\x\in\X}\epsilon\inner{\bm{L}}{\x}+\mathcal{R}(\x)$ and $\x'_* = \argmin_{\x\in\X}\epsilon\inner{\bm{L'}}{\x}+\mathcal{R}(\x)$ for a $\lambda$-strongly convex regularizer $\mathcal{R}:\mathcal{X}\mapsto \mathbb{R}$ with respect to a norm $\Vert\cdot\Vert$ and some $\bm{L}\in\mathbb{R}^d$ and $\bm{L}'\in\mathbb{R}^d$. Then we have
\begin{equation}
\label{eq:stability-FTRL}
\lambda\Vert\x_*-\x'_*\Vert\leq\epsilon\Vert\bm{L}-\bm{L}'\Vert_\star,
\end{equation}
where $\Vert\cdot\Vert_\star$ is the dual norm of $\Vert\cdot\Vert$, defined as $\norm{\y}_{\star}:=\sup \{\inner{\x}{\y} \mid \norm{\x} \leq 1\}$.
\end{myLemma}

\begin{proof}
Define $F(\x) = \epsilon\inner{\bm{L}}{\x}+\mathcal{R}(\x)$ and $F'(\x) = \epsilon\inner{\bm{L}'}{\x}+\mathcal{R}(\x)$. Clearly, both $F(\x)$ and $F'(\x)$ are $\lambda$-strongly convex. As a result, according to Lemma~\ref{lem:strongly-convex} we have
\begin{equation*}
  F(\x'_*)\geq F(\x_*) + \frac{\lambda}{2}\Vert\x'_*-\x_*\Vert^2,
\end{equation*}
and 
\begin{equation*}
  \label{eq:lem-proof-stability-FTRL-2}
  F'(\x_*)\geq F'(\x'_*) + \frac{\lambda}{2}\Vert\x'_*-\x_*\Vert^2.
\end{equation*}
Combining above two inequalities, we have
\begin{align*}
	\lambda \norm{\x'_*-\x_*}^2 \leq F(\x'_*) - F(\x_*) + F'(\x_*) - F'(\x'_*) = \epsilon\inner{\x'_*-\x_*}{\bm{L}-\bm{L}'}.
\end{align*}
By the Cauchy–Schwarz inequality, we further have
\[
  \lambda\norm{\x'_*-\x_*}^2 \leq \epsilon\inner{\x'_*-\x_*}{\bm{L}-\bm{L}'} \leq \epsilon \norm{\x'_*-\x_*}\cdot\norm{\bm{L}-\bm{L}'}_\star,
\]
which implies the desired result in the statement.
\end{proof}

We prove Lemma~\ref{lemma:OptimisticHedge} based on the above two lemmas.
\begin{proof}[Proof of Lemma~\ref{lemma:OptimisticHedge}] First, notice that the update procedure of the OptimisticHedge algorithm
\begin{equation*}
p_{t+1,i} \propto \exp\left(-\epsilon(L_{t,i} + m_{t+1,i})\right), \quad \forall i\in[N]
\end{equation*}
is essentially solving the following FTRL problem
\begin{equation}
\label{eq:lemma1-proof-step1}
  \p_t = \argmin_{\p\in\Delta_N}~\epsilon\inner{\Lb_t + \bm{m}_t}{\p} + \mathcal{R}(\p),
\end{equation}
where $\Lb_t = [L_{t,1},L_{t,2},\ldots,L_{t,N}]^\T \in \R^N$ is the loss vector, $\bm{m}_t = [m_{t,1},m_{t,2},\ldots,m_{t,N}]^\T \in \R^N$ is the optimistic vector, and $\mathcal{R}(\p) = \sum_{i\in[N]}p_{i}\ln p_{i}$ is a 1-strongly convex function with respect to $\Vert\cdot\Vert_1$. Thus, to prove Lemma~\ref{lemma:OptimisticHedge}, it is sufficient to analyze the property of update procedure~\eqref{eq:lemma1-proof-step1}. Actually, we can prove a more general result that for any comparator $\q\in\Delta_N$, the regret of the OptimisticHedge algorithm is at most
\begin{equation}
\label{eq:lemma-Optimistic-general}
\sum_{t=1}^T \inner{\ellb_t}{\p_t-\q}  \leq\frac{\ln N+ \mathcal{R}(\q)}{\epsilon} + \epsilon \sum_{t=1}^T \Vert\bm{\ell}_t-\bm{m}_t\Vert_{\infty} -\frac{1}{2\epsilon}\sum_{t=1}^T (\Vert\p_t-\p_t'\Vert_1^2+\Vert\p_t-\p_{t-1}'\Vert_1^2),
\end{equation}
where $\p_t$ is the return of the FTRL algorithm, and $\p_t'$ is the return of the ``Be The Leader (BTL)'' algorithm, whose exact formulations are shown below: 
\[
  \p_t = \argmin_{\p\in\Delta_N}~\epsilon\left\langle\sum_{\tau=1}^{t-1}\bm{\ell}_\tau+\bm{m}_t,\p\right\rangle+\mathcal{R}(\p)
\]
and
\[
  \p_t' = \argmin_{\p\in\Delta_N}~\epsilon\left\langle\sum_{\tau=1}^{t}\bm{\ell}_\tau, \p\right\rangle+\mathcal{R}(\p).
\]

Actually, suppose that the above argument~\eqref{eq:lemma-Optimistic-general} holds for any comparator $\q\in\Delta_N$, we can simply set $\q = \bm{e}_i$, the zero vector except that the $i$-th entry equals 1. Now $\mathcal{R}(\q) = 0$ and the last term can be further bounded as 
\begin{align*}
	    &\frac{1}{2\epsilon} \sum_{t=1}^T(\Vert\p_t-\p_t'\Vert_1^2 + \Vert\p_t-\p_{t-1}'\Vert_1^2)\\
\geq {} & \frac{1}{2\epsilon}\sum_{t=1}^{T}(\Vert\p_{t}-\p_{t}'\Vert_1^2 + \Vert\p_{t+1}-\p_{t}'\Vert_1^2)-\frac{1}{2\epsilon}\norm{\p_{T+1}-\p'_T}_1^2\\
\geq {} & \frac{1}{4\epsilon} \sum_{t=1}^{T-1}\Vert \p_{t+1}-\p_{t}\Vert_1^2 - \frac{2}{\epsilon}.
\end{align*}
The last inequality follows from the fact $(a+b)^2 \leq 2a^2+2b^2$ and the triangle inequality. Therefore, combining~\eqref{eq:lemma-Optimistic-general} and above arguments, we can finish the proof of Lemma~\ref{lemma:OptimisticHedge}, providing that~\eqref{eq:lemma-Optimistic-general} holds for any comparator $\q \in \Delta_N$. 

Now it suffices to prove~\eqref{eq:lemma-Optimistic-general}. We notice that the regret of the OptimisticHedge algorithm can be decomposed as
\begin{equation*}
\sum_{t=1}^T \inner{\ellb_t}{\p_t-\q} = \underbrace{\sum_{t=1}^T \inner{\ellb_t-\bm{m}_t}{\p_t-\p_t'}}_{\mathtt{term~(a)}} + \underbrace{\sum_{t=1}^T \inner{\bm{m}_t}{\p_t-\p_t'} + \sum_{t=1}^T \inner{\ellb_t}{\p_t'-\q}}_{\mathtt{term~(b)}},
\end{equation*}

According to Lemma~\ref{lem:stability-FTRL}, we know that $\norm{\p_t-\p_t'}_1 \leq \epsilon \sum_{t=1}^T\Vert\bm{\ell}_t-\bm{m}_t\Vert_{\infty}$ since $\mathcal{R}(\cdot)$ is 1-strongly convex with respect to $\Vert\cdot\Vert_1$-norm. Thereby, we achieve that
\begin{equation*}
  \begin{aligned}
  \mathtt{term~(a)} = {} & \sum_{t=1}^T \inner{\bm{\ell}_t-\bm{m}_t}{\p_t-\p_t'}\\
  \leq {} & \sum_{t=1}^T \norm{\bm{\ell}_t-\bm{m}_t}_{\infty}\cdot \norm{\p_t-\p_t'}_1 \qquad &(\text{by~the~H\"older's~Inequality})\\
  \leq {} & \epsilon \sum_{t=1}^T\Vert\bm{\ell}_t-\bm{m}_t\Vert_{\infty}^2.
  \end{aligned}
\end{equation*}

Then we proceed to prove term (b), more concretely, to prove the following result:
\begin{equation*}
  \sum_{t=1}^T \inner{\bm{m}_t}{\p_t-\p_t'} + \sum_{t=1}^T \inner{\bm{\ell}_t}{\p_t'-\q}\leq \frac{\ln N + \mathcal{R}(\q)}{\epsilon} -\frac{1}{2\epsilon}\sum_{t=1}^T \big(\Vert\p_t-\p_t'\Vert_1^2+\Vert\p_t-\p_{t-1}'\Vert_1^2 \big).
\end{equation*}
It turns out that the above inequality can be proved by induction: the base case (when $T=0$) holds apparently because of $\mathcal{R}(\q)>-\ln N$. Suppose the above inequality holds at iteration $T$, we show that it is also satisfied at iteration $T+1$ for all $\q\in\Delta_N$. 

Denoting $A_T = \frac{1}{2} \sum_{t=1}^T (\Vert\p_t-\p_t'\Vert_1^2+\Vert\p_t-\p_{t-1}'\Vert_1^2)$, we have
\begin{align*}
  &\sum_{t=1}^{T+1} \lg\bm{m}_t,\p_t-\p_t'\rg + \sum_{t=1}^{T+1} \lg\bm{\ell}_t,\p_t'\rg\\
  \leq {}&\lg\bm{m}_{T+1},\p_{T+1}-\p_{T+1}'\rg + \lg\bm{\ell}_{T+1},\p_{T+1}'\rg +\frac{\ln N + \mathcal{R}(\p'_{T})-A_T}{\epsilon} + \sum_{t=1}^{T}\inner{\bm{\ell}_t}{\p'_{T}}\\
  \leq {}&\lg\bm{m}_{T+1},\p_{T+1}-\p_{T+1}'\rg + \lg\bm{\ell}_{T+1},\p_{T+1}'\rg +\frac{\ln N + \mathcal{R}(\p_{T+1})-A_T-\frac{1}{2}\Vert\p_{T+1}-\p'_T\Vert_1^2}{\epsilon}+\sum_{t=1}^{T}\inner{\bm{\ell}_t}{\p_{T+1}}\\
  = {}&\lg\bm{\ell}_{T+1}-\bm{m}_{T+1},\p_{T+1}'\rg +\frac{\ln N + \mathcal{R}(\p_{T+1})-A_T-\frac{1}{2}\Vert\p_{T+1}-\p'_T\Vert_1^2}{\epsilon} + \sum_{t=1}^{T} \inner{\bm{\ell}_t+\bm{m}_{T+1}}{\p_{T+1}}\\
  \leq {}&\lg\bm{\ell}_{T+1}-\bm{m}_{T+1},\p_{T+1}'\rg +\frac{\ln N +\mathcal{R}(\p'_{T+1})-A_{T+1}}{\epsilon} + \sum_{t=1}^{T}\inner{\bm{\ell}_t+\bm{m}_{T+1}}{\p'_{T+1}}\\
  = {}&\frac{\ln N +\mathcal{R}(\p'_{T+1})-A_{T+1}}{\epsilon} + \sum_{t=1}^{T+1}\inner{\bm{\ell}_{t}}{\p'_{T+1}}\\
  \leq {} & \frac{\ln N + \mathcal{R}(\q)-A_{T+1}}{\epsilon} + \sum_{t=1}^{T+1}\inner{\bm{\ell}_{t}}{\q}.
\end{align*}
The first inequality holds by the induction assumption and setting $\q=\p'_{T}$. The second inequality holds by~\eqref{eq:strongly-convex} and that $F_T(\p) = \epsilon \sum_{t=1}^T\inner{\bm{\ell}_t}{\p}+\mathcal{R}(\p)$ is 1-strongly convex with respect to $\Vert\cdot\Vert_1$-norm as well as $\p'_T=\argmin_{\p\in\Delta_N}F_T(\p)$. The third inequality holds by the same argument as the second one and that $\p_{T+1}=\argmin_{\p\in\Delta_N}\epsilon \sum_{t=1}^{T} \inner{\bm{\ell}_t+\bm{m}_{T+1}}{\p}+\mathcal{R}(\p)$. The last inequality holds by the fact that $\p'_{T+1}=\argmin_{\p\in\Delta_N}\epsilon \sum_{t=1}^{T}\inner{\bm{\ell}_{T+1}}{\p} + \mathcal{R}(\p)$. 
\end{proof}

\section{Technical Lemmas}
\label{sec:appendix-tech-lemma}
In this part we present several technical lemmas used in the proofs. First, we introduce the self-bounding property of smooth functions~\citep[Lemma 3.1]{NIPS'10:smooth}, which is crucial and frequently used in proving problem-dependent bounds for convex and smooth functions.
\begin{myLemma} \label{lem:smooth} For an $L$-smooth and non-negative function $f: \X \mapsto \R_+$,
\[
\norm{\nabla f(\x)}_2 \leq \sqrt{4 L f(\x)}, \ \forall \x \in \X.
\]
\end{myLemma}
From the analysis of~\citep[Lemma 2.1 and Lemma 3.1]{NIPS'10:smooth}, we can find that actually the non-negativity is required outside the domain $\X$, and this is why we require the function $f_t(\cdot)$ to be non-negative outside the domain $\X$ in Assumption~\ref{assumption:non-negative}.

\begin{myLemma}[Lemma 19 of~\citet{thesis:shai2007}]
\label{lemma:inquality-shai}
For any $x,y,a \in \R_+$ that satisfy $x-y \leq \sqrt{ax}$,
\begin{equation}
  \label{eq:inquality}
  x-y \leq a+\sqrt{ay}.
\end{equation}
\end{myLemma}

Based on Lemma~\ref{lemma:inquality-shai}, we have the following result.
\begin{myLemma}
\label{cor:inquality-shai}
For any $x,y,a, b\in \R_+$ that satisfy $x-y \leq \sqrt{ax} + b$, 
\begin{equation}
  \label{eq:inquality-cor}
  x-y \leq a + b + \sqrt{ay + ab}
\end{equation}
\end{myLemma}

The following projection lemma is useful in analyzing the gradient descent algorithm.
\begin{myLemma}
  \label{lemma:bregman-divergence}
  Let $\mathcal{X}$ be a closed convex set. Then, any update of the form $\x^* = \Pi_{\mathcal{X}}[\mathbf{c}-\nabla]$ 
  satisfies the following inequality
  \begin{equation}
    \label{eq:bregman-divergence}
    \langle \x^* - \u, \nabla\rangle \leq \frac{1}{2}\norm{\mathbf{c}-\u}_2^2 - \frac{1}{2}\norm{\x^*-\u}_2^2 - \frac{1}{2}\norm{\x^*-\mathbf{c}}_2^2
  \end{equation}
  for any $\u \in \X$.
\end{myLemma}

\begin{proof}
It is equivalent to prove the following inequality
\begin{equation}
\label{eq:proof_lemma_breg}
  \langle  \mathbf{u} - \x^*, (\mathbf{c}-\nabla)-\x^*\rangle\leq0.
\end{equation}
We consider two cases by noting that $\x^* = \Pi_{\mathcal{X}}[\mathbf{c}-\nabla]$:
\begin{enumerate}
  \item[(1)] $\mathbf{c}-\nabla\in\mathcal{X}$: $\langle  \mathbf{u} - \x^*, (\mathbf{c}-\nabla)-\x^*\rangle=0$ clearly satisfies~\eqref{eq:proof_lemma_breg};
  \item[(2)] $\mathbf{c}-\nabla\notin \mathcal{X}$: the Pythagorean theorem~\citep[Theorem 2.1]{book'16:Hazan-OCO} implies~\eqref{eq:proof_lemma_breg}.
\end{enumerate}
This ends the proof. 
\end{proof}
\end{document}